\documentclass{article} % For LaTeX2e
\usepackage{iclr2024_conference,times}

\usepackage[linktoc=all]{hyperref}
\usepackage{url}
\usepackage{amsmath}
\usepackage{amsthm}
\usepackage{amssymb}
\usepackage{graphicx}
\usepackage{wrapfig}
\usepackage{minitoc}
\usepackage{etoc}

\newcommand{\eg}{\textit{e.g.}}
\newcommand{\ie}{\textit{i.e.}}
\newtheorem{theorem}{Theorem}

\newcommand{\R}{\mathbb{R}}
\DeclareMathOperator*{\E}{\mathbb{E}}
\newcommand{\dd}{d_\mathrm{data}}
\newcommand{\Md}{\mathcal{M}_{\dd}}
\newcommand{\dr}{d_\mathrm{repr}}
\newcommand{\KF}{K_\mathcal{F}}
\newcommand{\KFh}{\hat{K}_\mathcal{F}}

\newcommand{\Dtr}{\mathcal{D}_\mathrm{train}}
\newcommand{\Dts}{\mathcal{D}_\mathrm{test}}

\newcommand{\zz}{\hat{z}}
\newcommand{\xh}{\hat{x}}

\title{The Effect of Intrinsic Dataset Properties on Generalization: Unraveling Learning Differences Between Natural and Medical Images}

\author{Nicholas Konz$^{1}$, Maciej A. Mazurowski$^{1,2,3,4}$ \\
$^{1}$ Department of Electrical and Computer Engineering, $^{2}$ Department of Radiology, \\
$^{3}$ Department of Computer Science, $^{4}$ Department of Biostatistics \& Bioinformatics \\
Duke University, NC, USA \\
\texttt{\{nicholas.konz, maciej.mazurowski\}@duke.edu} \\
}

\iclrfinalcopy % Uncomment for camera-ready version, but NOT for submission.
\begin{document}

\maketitle

\begin{abstract}
This paper investigates discrepancies in how neural networks learn from different imaging domains, which are commonly overlooked when adopting computer vision techniques from the domain of natural images to other specialized domains such as medical images. Recent works have found that the generalization error of a trained network typically increases with the intrinsic dimension ($\dd$) of its training set. Yet, the steepness of this relationship varies significantly between medical (radiological) and natural imaging domains, with no existing theoretical explanation. We address this gap in knowledge by establishing and empirically validating a generalization scaling law with respect to $\dd$, and propose that the substantial scaling discrepancy between the two considered domains may be at least partially attributed to the higher intrinsic ``label sharpness'' ($\KF$) of medical imaging datasets, a metric which we propose.
Next, we demonstrate an additional benefit of measuring the label sharpness of a training set: it is negatively correlated with the trained model's adversarial robustness, which notably leads to models for medical images having a substantially higher vulnerability to adversarial attack. 
Finally, we extend our $\dd$ formalism to the related metric of learned representation intrinsic dimension ($\dr$), derive a generalization scaling law with respect to $\dr$, and show that $\dd$ serves as an upper bound for $\dr$. Our theoretical results are supported by thorough experiments with six models and eleven natural and medical imaging datasets over a range of training set sizes. Our findings offer insights into the influence of intrinsic dataset properties on generalization, representation learning, and robustness in deep neural networks.\footnote{Code link: \url{https://github.com/mazurowski-lab/intrinsic-properties}.}
\end{abstract}

\section{Introduction}

There has been recent attention towards how a neural network's ability to generalize to test data relates to the \textit{intrinsic dimension} $\dd$ of its training dataset, \ie, the dataset's inherent ``complexity'' or the minimum degrees of freedom needed to represent it without substantial information loss \citep{gong2019intrinsic}. Recent works have found that generalization error typically increases with $\dd$, empirically \citep{pope2021intrinsic} or theoretically \citep{bahri2021explaining}.
Such ``scaling laws'' with respect to intrinsic dataset properties are attractive because they may describe neural network behavior in generality, for different models and/or datasets, allowing for better understanding and predictability of the behavior, capabilities, and challenges of deep learning.
%However, the \textit{rate} of this $\dd$ generalization scaling behavior differs drastically between natural and medical image datasets \citep{konz2022intrinsic}, with no clear theoretical explanation. 
%It is crucial to unravel why and how neural network performance differs between these two key imaging domains, so that methods designed for natural images are not na\"ively applied to medical image analysis tasks. 
However, a recent study \citep{konz2022intrinsic} showed that generalization scaling behavior differs drastically depending on the input image type, \eg, natural or medical images, showing the non-universality of the scaling law and motivating us to consider its relationship to properties of the dataset and imaging domain.\footnote{Here we take ``medical'' images to refer to radiology images (\eg, x-ray, MRI), the most common type.}

In this work, we provide theoretical and empirical findings on how measurable intrinsic properties of an image dataset can affect the behavior of a neural network trained on it. We show that certain dataset properties that differ between imaging domains can lead to discrepancies in behaviors such as generalization ability and adversarial robustness. Our contributions are summarized as follows.

First, we introduce the novel measure of the intrinsic \textbf{label sharpness} ($\KF$) of a dataset (defined in Section \ref{sec:KFest}).
The label sharpness essentially measures how similar images in the dataset can be to each other while still having different labels, and we find that it usually differs noticeably between natural and medical image datasets.
We then derive and test a neural network generalization scaling law with respect to dataset intrinsic dimension $\dd$, which includes $\KF$. Our experiments support the derived scaling behavior within each of these two domains, and show a distinct difference in the scaling rate between them. According to our scaling law and likelihood analysis of observed generalization data (Appendix C.1), this may be due to the measured $\KF$ being typically higher for medical datasets.

Next, we show how a model's adversarial robustness relates to its training set's $\KF$, and show that over a range of attacks, robustness decreases with higher $\KF$. Indeed, medical image datasets, which have higher $\KF$, \textbf{are typically more susceptible to adversarial attack than natural image datasets}. 
Finally, we extend our $\dd$ formalism to derive and test a generalization scaling law with respect to the intrinsic dimension of the model's \textit{learned representations}, $\dr$, and reconcile the $\dd$ and $\dr$ scaling laws to show that $\dd$ serves as an approximate upper bound for $\dr$. 
We also provide many additional results in the supplementary material, such as a likelihood analysis of our proposed scaling law given observed generalization data (Appendix C.1), the evaluation of a new dataset in a third domain (Appendix C.2), an example of a practical application of our findings (Appendix C.3), and more.

All theoretical results are validated with thorough experiments on six CNN architectures and eleven datasets from natural and medical imaging domains over a range of training set sizes. 
We hope that our work initiates further study into how network behavior differs between imaging domains.

\section{Related Works}

We are interested in the scaling of the generalization ability of supervised convolutional neural networks with respect to intrinsic properties of the training set. Other works have also explored generalization scaling with respect to parameter count or training set size for vision or other modalities \citep{caballero2023broken, kaplan2020scaling, hoffmann2022training, touvron2023llama}.
Note that we model the intrinsic dimension to be constant throughout the dataset's manifold as in \citet{pope2021intrinsic, bahri2021explaining} for simplicity, as opposed to the recent work of \citet{brown2023verifying}, which we find to be suitable for interpretable scaling laws and dataset properties.

% \citet{ghorbani2020neural} found that trained image models typically dependent on the same low-dimensional subspace as their dataset.

Similar to  \textit{dataset} intrinsic dimension scaling \citep{pope2021intrinsic, bahri2021explaining, konz2022intrinsic}, recent works have also found a monotonic relationship between a network's generalization error and the intrinsic dimension of both the learned hidden layer representations \citep{ansuini2019intrinsic}, or some measure of intrinsic dimensionality of the trained model itself \citep{birdal2021intrinsic, andreeva2023metric}. In this work, we focus on the former, as the latter model dimensionality measures are typically completely different mathematical objects than the intrinsic dimension of the manifolds of data or representations. Similarly, \citet{kvinge2023exploring} found a correlation between prompt perplexity and representation intrinsic dimension in Stable Diffusion models.

\section{Preliminaries}
\label{sec:preliminaries}

We consider a binary classification dataset $\mathcal{D}$ of points $x\in\R^n$ with target labels $y=\mathcal{F}(x)$ defined by some unknown function $\mathcal{F}:\R^n\rightarrow \{0,1\}$, split into a training set $\Dtr$ of size $N$ and test set $\Dts$. The manifold hypothesis \citep{fefferman2016manifoldhyp} assumes that the input data $x$ lies approximately on some $\dd$-dimensional manifold $\Md\subset\R^n$, with $\dd\ll n$. More technically, $\Md$ is a metric space such that for all $x\in\Md$, there exists some neighborhood $U_x$ of $x$ such that $U_x$ is homeomorphic to $\R^{\dd}$, defined by the standard $L_2$ distance metric $||\cdot||$.
% (in other words, $\Md$ is locally flat in $\dd$ dimensions).

As in \citet{bahri2021explaining}, we consider over-parameterized (number of parameters $\gg N$) models $f(x):\R^n\rightarrow \{0, 1\}$, that are ``well-trained'' and learn to interpolate all training data: $f(x)=\mathcal{F}(x)$ for all $x\in \Dtr$. We use a non-negative loss function $L$, such that $L=0$ when $f(x)=\mathcal{F}(x)$. Note that we write $L$ as the expected loss over a set of test set points. We assume that $\mathcal{F}$, $f$ and $L$ are Lipschitz/smooth on $\Md$ with respective constants $\KF$, $K_f$ and $K_L$. Note that we use the term ``\textit{Lipschitz constant}'' of a function to refer to the smallest value that satisfies the Lipschitz inequality.\footnote{A subtlety here is that our Lipschitz assumptions only involve pairs of datapoints sampled from the true data manifold $\Md$; adversarially-perturbed images \citep{DBLP:journals/corr/GoodfellowSS14} are not included.}
We focus on binary classification as in \citet{pope2021intrinsic, konz2022intrinsic}, but we note that our results extend naturally to the multi-class case (see Appendix A.1 for more details).
% for more details. Before proceeding further, we will first introduce how the intrinsic dimension $\dd$ and Lipschitz constant $\KF$, or \textit{label sharpness}, of a dataset can be estimated, which we will use throughout this work.

\subsection{Estimating Dataset Intrinsic Dimension}
\label{sec:dest}
Here we introduce two common intrinsic dimension estimators for high-dimensional datasets that we use in our experiments, which have been used in prior works on image datasets \citep{pope2021intrinsic, konz2022intrinsic} and learned representations \citep{ ansuini2019intrinsic, gong2019intrinsic}.

\textbf{MLE:} The MLE (maximum likelihood estimation) intrinsic dimension estimator \citep{levina2004maximum, mackay2005comments} works by assuming that the number of datapoints enclosed within some $\epsilon$-ball about some point on $\Md$ scales not as $\mathcal{O}(\epsilon^n)$, but $\mathcal{O}(\epsilon^{\dd})$, and then solving for $\dd$ with MLE after modeling the data as sampled from a Poisson process. This results in $\hat{d}_{\mathrm{data}}=\left[\frac{1}{N(k-1)} \sum_{i=1}^{N} \sum_{j=1}^{k-1} \log \frac{T_{k}\left(x_{i}\right)}{T_{j}\left(x_{i}\right)}\right]^{-1}$, where $T_j(x)$ is the $L_2$ distance from $x$ to its $j^{th}$ nearest neighbor and $k$ is a hyperparameter; we set $k=20$ as in \citet{pope2021intrinsic, konz2022intrinsic}.
\textbf{TwoNN:} TwoNN \citep{facco2017estimating} is a similar approach that instead relies on the ratio of the first- and second-nearest neighbor distances. We default to using the MLE method for $\dd$ estimation as \citet{pope2021intrinsic} found it to be more reliable for image data than TwoNN, but we still evaluate with TwoNN for all experiments. Note that these estimators do not use datapoint labels.

\subsection{Estimating Dataset Label Sharpness}
\label{sec:KFest}
Another property of interest is an empirical estimate for the ``label sharpness'' of a dataset, $\KF$. This measures the extent to which images in the dataset can resemble each other while still having different labels. Formally, $\KF$  is the Lipschitz constant of the ground truth labeling function $\mathcal{F}$, $\ie$, the smallest positive $\KF$ that satisfies $\KF ||x_1 - x_2|| \geq |\mathcal{F}(x_1) - \mathcal{F}(x_2)| =|y_1 - y_2|$ for all $x_1,x_2\sim \Md$, where $y_i=\mathcal{F}(x_i)\in\{0,1\}$ is the target label for $x_i$. We estimate this as
\begin{equation}
\label{eq:KFest}
     \KFh := \max\limits_{j, k} \left(\frac{|y_j-y_k|}{|| x_j - x_k||}\right),
\end{equation}
computed over all $M^2$ pairings $((x_j, y_j), (x_k, y_k))$ of some $M$ evenly class-balanced random samples $\left\{(x_i, y_i) \right\}_{i=1}^M$ from the dataset $\mathcal{D}$. We use $M=1000$ in practice, which we found more than sufficient for a converging estimate, and it takes $<$1 sec. to compute $\KFh$.
We minimize the effect of trivial dataset-specific factors on $\KFh$ by linearly normalizing all images to the same range (Sec. \ref{sec:data}), and we note that both $\KFh$ and $\dd$ are invariant to image resolution and channel count (Appendix B.1). As the natural image datasets have multiple possible combinations of classes for the binary classification task, we report $\KFh$ averaged over 25 runs of randomly chosen class pairings.
% \suggest{should this be median instead of max in practice? max can be an unstable estimate, but makes more sense for this}

\section{Datasets, Models and Training}
\label{sec:data}

\paragraph{Medical Image Datasets.}
We conducted our experiments on seven public medical image (radiology) datasets from diverse modalities and anatomies for different binary classification tasks. These are (1) brain MRI glioma detection (\textbf{BraTS}, \citet{menze2014multimodal}); (2) breast MRI cancer detection (\textbf{DBC}, \citet{saha2018machinedukedbc}); (3) prostate MRI cancer risk scoring (\textbf{Prostate MRI}, \citet{sonn2013prostate}); (4) brain CT hemorrhage detection (\textbf{RSNA-IH-CT}, \citet{flanders2020rsnaihct}); (5) chest X-ray pleural effusion detection (\textbf{CheXpert}, \citet{irvin2019chexpert}); (6) musculoskeletal X-ray abnormality detection (\textbf{MURA}, \citet{rajpurkar2017mura}); and (7) knee X-ray osteoarthritis detection (\textbf{OAI}, \citet{Tiulpin2018}). All dataset preparation and task definition details are provided in Appendix G.

\paragraph{Natural Image Datasets.}
We also perform our experiments using four common ``natural'' image classification datasets: \textbf{ImageNet} \citep{deng2009imagenet}, \textbf{CIFAR10} \citep{krizhevsky2009cifar}, \textbf{SVHN} \citep{netzer2011svhn}, and \textbf{MNIST} \citep{deng2012mnist}.

For each dataset, we create training sets of size $N\in\{500, 750, 1000, 1250, 1500, 1750\}$, along with a test set of $750$ examples. These splits are randomly sampled with even class-balancing from their respective base datasets. For the natural image datasets we choose two random classes (different for each experiment) to define the binary classification task, and all results are averaged over five runs using different class pairs.\footnote{$N=1750$ is the upper limit of $N$ that all datasets could satisfy, given the smaller size of medical image datasets and ImageNet's typical example count per class. In Appendix C.4 we evaluate much higher $N$ for datasets that allow for it.} Images are resized to $224\times224$ and normalized linearly to $[0, 1]$.

\begin{figure}[!htbp]
\centering
\includegraphics[width=0.54\textwidth]{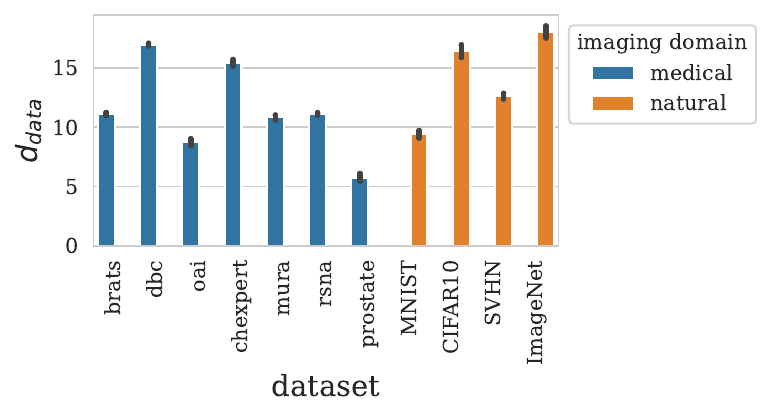}
\includegraphics[width=0.41\textwidth]{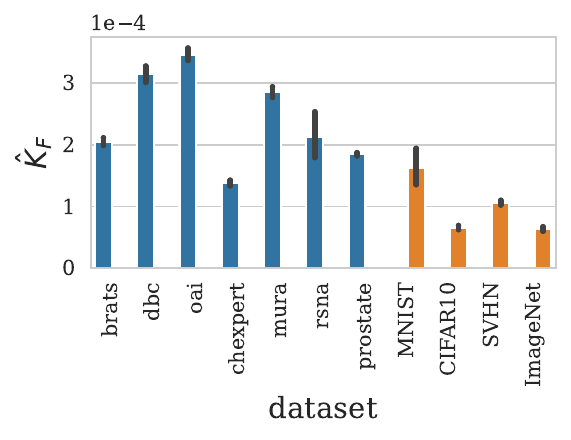}
    \caption{Measured intrinsic dimension ($\dd$, \textbf{left}) and label sharpnesses ($\KFh$, \textbf{right}) of the natural (orange) and medical (blue) image datasets which we analyze (Sec. \ref{sec:data}). \textbf{$\KFh$ is typically higher for the medical datasets}. 
    % $\KFh$ and $\dd$ characterize the two imaging domains, which we will show dictate different facets of learning behavior. 
    $\dd$ values are averaged over all training set sizes, and $\KFh$ over all class pairings (Sec. \ref{sec:KFest}); error bars indicate $95\%$ confidence intervals.}
    \label{fig:kf_vs_dd}
\end{figure}

\paragraph{Models and training.} We evaluate six models total: ResNet-18, -34 and -50 \citep{he2016resnet}, and VGG-13, -16 and -19 \citep{VGGs}. Each model $f$ is trained on each dataset for its respective binary classification task with Adam \citep{adam} until the model fully fits to the training set, for each training set size $N$ described previously.
We provide all training and implementation details in Appendix F, and our code can be found at \url{https://github.com/mazurowski-lab/intrinsic-properties}.

%\section{The relationship of model generalization ability with dataset intrinsic dimension and label sharpness}
\section{The Relationship of Generalization with Dataset Intrinsic Dimension and Label Sharpness}
% In this section we will analyze how the generalization ability / test set performance of the trained model $f$ scales with two measurable quantities characterizing the training set: intrinsic dimension (i.d.) $\dd$ and empirical label sharpness $ \KFh$, beginning with theoretical results and then with thorough experiments.

In Fig. \ref{fig:kf_vs_dd} we show the average measured intrinsic dimension $\dd$ and label sharpness $\KFh$ of each dataset we study.
% , with $\dd$ values averaged over all training set sizes, and $\KFh$ over all class pairings (Sec. \ref{sec:KFest}). 
While both natural and medical datasets can range in $\dd$, we note that medical datasets typically have much higher $\KFh$ than natural image datasets, which we will propose may explain differences in generalization ability scaling rates between the two imaging domains. 
We emphasize that $\dd$ and $\KF$ are model-independent properties of a dataset itself.
% and are also invariant to image resolution or channel count (Appendix \ref{sec:app:resize}).
% \footnote{We confirmed that $\dd$ and $\KFh$ are unaffected by changing image resolution or channel count (\eg, switching RGB$\Rightarrow$grayscale), beyond all datasets' $\KFh$ values being multiplied by a constant (Appendix \ref{sec:app:resize}).}
We will now describe how network generalization ability scales with $\dd$ and $\KF$.

\subsection{Bounding generalization ability with dataset intrinsic dimension}
\label{sec:exp:ddata_scaling}

A result which we will use throughout is that on average, given some $N$ datapoints sampled i.i.d. from a $d$-dimensional manifold, the distance between the nearest neighbor $\xh$ of some datapoint $x$ scales as $\E_{x}||x-\hat{x}|| = \mathcal{O}(N^{-1/\dd})$ \citep{levina2004maximum}. As such, the nearest-neighbor distance of some test point to the training set decreases as the training set grows larger by $\mathcal{O}(N^{-1/\dd})$. It can then be shown that the loss on the test set/generalization error scales as $\mathcal{O}(K_L\max (K_f, \KF)N^{-1/\dd})$ on average; this is summarized in the following theorem.

\begin{theorem}[Generalization Error and Dataset Intrinsic Dim. Scaling Law \citep{bahri2021explaining}]
  \label{thm:scaling_dd}
  Let $L$, $f$ and $\mathcal{F}$ be Lipschitz on $\Md$ with respective constants $K_L$, $K_f$ and $\KF$. Further let $\Dtr$ be a training set of size $N$ sampled i.i.d. from $\Md$, with $f(x)=\mathcal{F}(x)$ for all $x \in \Dtr$. Then, $L=\mathcal{O}(K_L\max (K_f, \KF)N^{-1/\dd})$.
\end{theorem}

% This implies that the loss has an upper bound that depends on the dataset's intrinsic dimension of $L\leq CK_L(K_f + \KF)N^{-1/\dd}$, where $C$ is a constant. 
We note that the $\KF$ term is typically treated as an unknown constant in the literature \citep{bahri2021explaining}; instead, we propose to \textit{estimate} it with the empirical label sharpness $\KFh$ (Sec. \ref{sec:KFest}). We will next show that $K_f\simeq \KF$ for large $N$ (common for deep models), which allows us to approximate Theorem \ref{thm:scaling_dd} as $L \simeq \mathcal{O}(K_L\KF N^{-1/\dd})$, \textbf{a scaling law independent of the trained model} $f$. Intuitively, this means that the Lipschitz smoothness of $f$ molds to the smoothness of the label distribution as the training set grows larger and test points typically become closer to training points. 
% This is nontrivial because $K_f$ and $\KF$ characterize the Lipschitz continuity for $f$ and $\mathcal{F}$ on all possible points on $\Md$, not just on the training set $\Dtr$ where $f(x)=\mathcal{F}(x)$. 
% This result will allow us to create approximate model generalization scaling laws that only depend on training set quantities $\KF$, $\dd$ and $N$.

\begin{theorem}[Approximating $K_f$ with $K_\mathcal{F}$]
  \label{thm:KfsimKF}
  $K_f$ converges to $\KF$ in probability as $N\rightarrow\infty$.
\end{theorem}

We show the full proof in Appendix A.2 due to space constraints.
% Recall that we assumed the model is ``well-trained'': $f(x)=\mathcal{F}(x)$ for all $x\in\Dtr$. This implies that $K_f=\KF$ \suggest{confirm/prove that this is true, as these lipschitz consts are defined by inequalities}, 
% As shown in the proof, $K_f\simeq \KF$ is a good approximation for the typically high values of $N$ and low values of $\dd$ seen in image datasets \citep{pope2021intrinsic,konz2022intrinsic}. 
This result is also desirable because computing an estimate for $K_f$, the Lipschitz constant of the model $f$, either using Eq. \eqref{eq:KFest} or with other techniques \citep{fazlyab2019efficient}, depends on the choice of model $f$, and may require many forward passes. Estimating $\KF$ (Eq. \eqref{eq:KFest} is far more tractable, as it is an intrinsic property of the dataset itself which is relatively fast to compute.

% $K_L \geq|y - \hat{y}| / \ell(y)$ for some test point prediction values $y$ with ground truth $\hat{y}$, which are not explicitly dependent on the dataset 

% describes the maximum sharpness of slope between a model's output and the ground truth label loss of a test point vs. difference in prediction vs test point ground truth, which wouldn't necessarily be domain-dependent.
 % such that we can approximate ..., with slope proportional to $K_L(K_f + \KF)$ for fixed $D$. 

% \suggest{but what about the other Lipschitz terms; how would these change between domains? if not properly considered, the could confound. In terms of $K_f$ vs $\KF$, we assume that the model trains well enough for $f=\mathcal{F} \forall x$ in the training set (not test set), therefore $K_f=\KF$, so the slope will be proportional to $K_L\KF$  Next, $K_L$ just describes the maximum sharpness of slope between loss of a test point vs. difference in prediction vs test point ground truth, which wouldn't necessarily be domain-dependent.}
Next, note that the Lipschitz constant $K_L$ is a property of the loss function, which we take as fixed \textit{a priori}, and so does not vary between datasets or models. As such, $K_L$ can be factored out of the scaling law of interest, such that we can simply consider $L\simeq \mathcal{O}(\KF N^{-1/\dd})$, \ie, 
\begin{equation}
\label{eq:loss_scaling_dd}
\log L \lesssim -\frac{1}{\dd}\log N + \log \KF + a
\end{equation}
for some constant $a$.
In the following section, we will demonstrate how the prediction of Eq. \eqref{eq:loss_scaling_dd} may explain recent empirical results in the literature where the rate of this generalization scaling law differed drastically between natural and medical datasets, via the measured differences in the typical label sharpness $\KFh$ of datasets in these two domains.

\subsection{Generalization Discrepancies Between Imaging Domains}
\label{sec:theory_lipschitz}

% We begin by examining the three Lipschitz constants of Thm. \ref{sec:theory_lipschitz}, on $L$, $f$ and $\mathcal{F}$, of $K_L$, $K_f$ and $\KF$, respectively.

% From \citet{bahri2021explaining}, the first examines the model's predicted target for some test point $x$ compared to the true target $\mathcal{F}(x)$: $\ell(f(x)) \leq K_L |f(x) - \mathcal{F}(x)| $, where $x$ is some test point and $\ell$ is the loss evaluated at the single $x$, using the aforementioned assumption of the loss vanishing on the true target.

% Next, consider some training point $x$ with nearest neighbor $\xh$ also in the training set. Now, note that the two other Lipschitz constants here are defined by the constraints $|f(x) - f(\xh)|\leq K_f|x - \xh|$ and $|\mathcal{F}(x) - \mathcal{F}(\xh)|\leq \KF|x - \xh|$,  While the smoothness of $f$ and $\mathcal{F}$ does not directly affect $K_L$, 

% If we take the loss bound in Equation \eqref{eq:scaling} as an estimate as expected by \citet{bahri2021explaining}, we have that $L\propto K_L(K_f + \KF)N^{-1/\dd}$.
Consider the result from Eq. \eqref{eq:loss_scaling_dd} that the test loss/generalization error scales approximately as $L\propto \KF N^{-1/\dd}$ on average.
% The label sharpness $\KF$ of the dataset describes how similar its images can be while still having different labels. In other words, if the data is sampled from some domain where a slight change to the image can change its label, $\KF$ will be higher. 
From this, we hypothesize that a higher label sharpness $\KF$ will result in the test loss curve that grows faster with respect to $\dd$.
% : \textbf{ a dataset with the same intrinsic dimension $\dd$ as another but with higher label sharpness $\KF$ will have worse generalization ability}.

In Fig. \ref{fig:loss_vs_datadim} we evaluate the generalization error (log test loss) scaling of all models trained on each natural and medical image dataset with respect to the training set intrinsic dimension $\dd$, for all evaluated training set sizes $N$. 
% Note that each point is for a single model trained on one of the datasets sampled at one of the training sizes $N$. 
We also show the scaling of test \textit{accuracy} in Appendix E.1.

\begin{figure}[!htbp]
\centering
\includegraphics[width=0.98\textwidth]{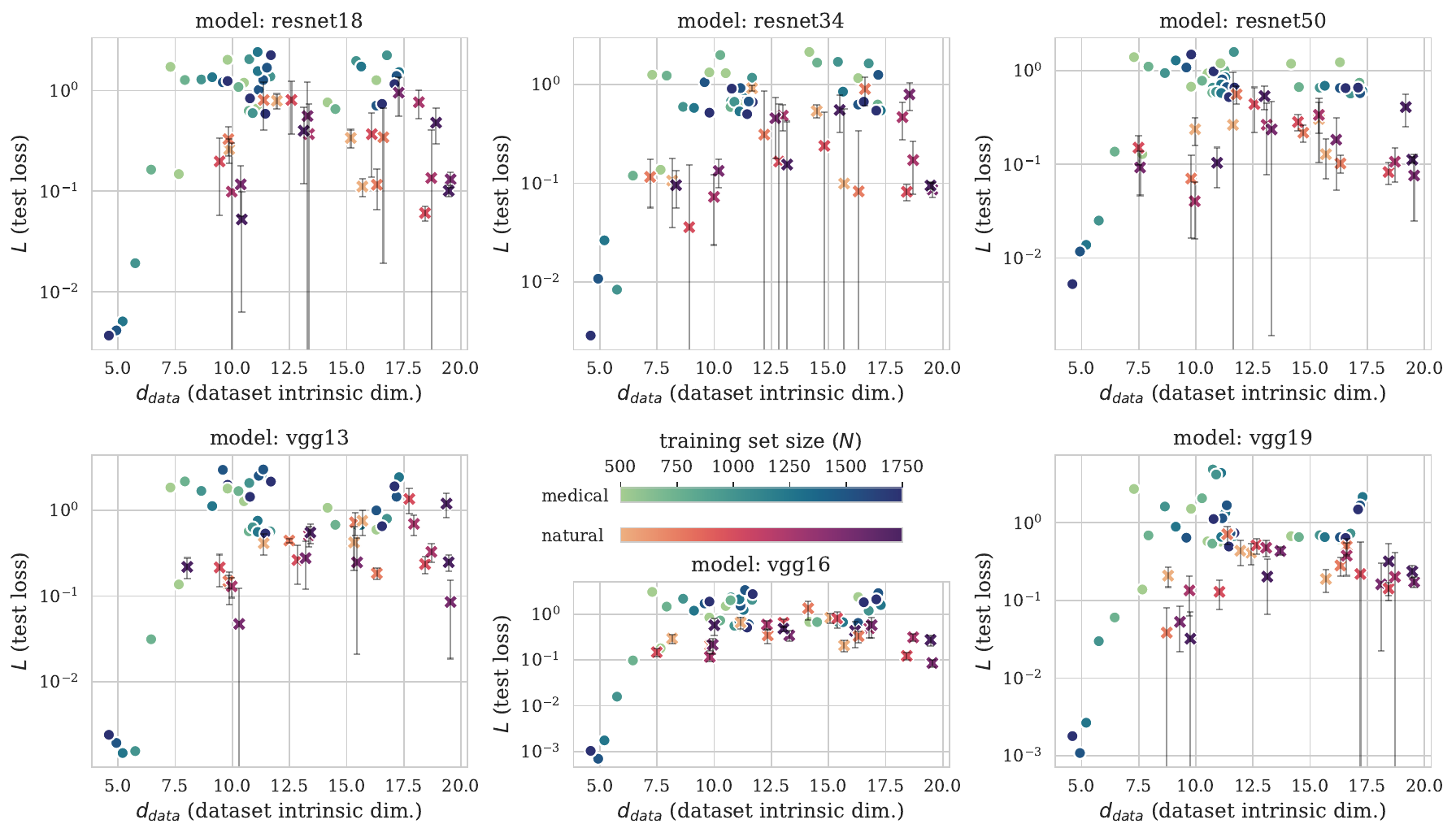}
\caption{Scaling of log test set loss/generalization ability with training dataset intrinsic dimension ($\dd$) for natural and medical datasets. Each point corresponds to a (model, dataset, training set size) triplet. Medical dataset results are shown in blue shades, and natural dataset results are shown in red; note the difference in generalization error scaling rate between the two imaging domains. Standard deviation error bars are shown for natural image datasets for 5 different class pairs.% Shaded regions for each domain (\textit{for illustrative purposes only}) plot the generalization bound of Eq. \eqref{eq:loss_scaling_dd} evaluated with the highest measured dataset $\KF$ value of each domain, and plausible value $a=9.8$.
% \textbf{Bottom:} same but for test accuracy.
}
\label{fig:loss_vs_datadim}
\end{figure}

We see that \textit{within} an imaging domain (natural or medical), model generalization error typically increases with $\dd$, as predicted, similar to prior results \citep{pope2021intrinsic,konz2022intrinsic}; in particular, approximately $\log L\propto -1/\dd + \mathrm{const.}$, aligning with Eq. \eqref{eq:loss_scaling_dd}. 
% The loss also typically reduces with higher training set size $N$, which aligns with Eq. \eqref{eq:loss_scaling_dd} and intuitions of sample complexity \citep{narayanan2010sample}.
However, we also see that the generalization error scaling is much sharper for models trained on medical data than natural data; models trained on datasets with similar $\dd$ and of the same size $N$ tend to perform much worse if the data is medical images. A similarly large gap appears for the scaling of test accuracy (Appendix E.1). We posit that this difference is explained by medical datasets typically having much higher label sharpness ($\KFh \sim2.5\times 10^{-4}$) than natural images ($\KFh\sim 1\times 10^{-4}$) (Fig. \ref{fig:kf_vs_dd}) 
% (even if the intrinsic dimension is similar)
, as $\KF$ is the only term in  Eq. \eqref{eq:loss_scaling_dd} that differs between two models with the same training set intrinsic dimension $\dd$ and size $N$. Moreover, in Appendix C.1 we show that accounting for $\KF$ increases the likelihood of the posited scaling law given the observed generalization data. However, we note that there could certainly be other factors causing the discrepancy which are not accounted for. 
% Finally, we also see that models with higher training set size $N$ typically have better generalization ability, which is intuitively reasonable and aligns with the predictions of Eq. \eqref{eq:loss_scaling_dd}.

Intuitively, the difference in dataset label sharpness $\KF$ between these imaging domains is reasonable, as $\KF$ describes how similar a dataset's images can be while still having different labels (Sec. \ref{sec:KFest}).
% In other words, if the data is sampled from some domain where a slight change to the image can change its label, $\KF$ will be higher. 
For natural image classification, images from different classes are typically quite visually distinct. However, in many medical imaging tasks, a change in class can be due to a small change or abnormality in the image, resulting in a higher dataset $\KF$; for example, the presence of a small breast tumor will change the label of a breast MRI from healthy to cancer. 
% This maximum possible proportion of change in the label $y=\mathcal{F}(x)$ to the change in the image is exactly the label sharpness $\KF$. 

% \begin{wrapfigure}[15]{r}{0.5\textwidth}
%     \vspace{-4.75mm}
%     \includegraphics[width=0.5\textwidth]{figs/kf_vs_data_dim.pdf}\vspace{-5mm}
%     \caption{Measured average label sharpness ($ \KFh$) and intrinsic dimension ($\dd$) of natural and medical datasets \suggest{should show error bars for these}. \suggest{in a way, this plot characterizes these two imaging domains with useful intrinsic characteristics}}\vspace{-3mm}
%     \label{fig:kf_vs_dd}
%     \vspace{-3mm}
% \end{wrapfigure}

% \suggest{What about $K_f$ and $K_L$? How do these change between datasets? Does $K_f=\KF$?} 

% In other words, $\KF$ will be higher for a training set where an image can be slightly changed (to another image in the training set) and yet drastically affect the target value, and vice versa.

% In the following section, we will connect this analysis to the intrinsic dimension of models' \textit{learned representations}.

\section{Adversarial Robustness and Training Set Label Sharpness}
\label{sec:exp:robust}

In this section we present another advantage of obtaining the sharpness of the dataset label distribution ($\KF$): it is negatively correlated with the adversarial robustness of a neural network. Given some test point $x_0\in\Md$ with true label $y=\mathcal{F}(x_0)$, the general goal of an adversarial attack is to find some $\tilde{x}$ that appears similar to $x_0$ --- \ie, $||\tilde{x}-x_0||_\infty$ is small --- that results in a different, seemingly erroneous network prediction for $\tilde{x}$. Formally, the \textit{robustness radius} of the trained network $f$ at $x_0$ is defined by
\begin{equation}
    \label{eq:robustradius}
    % \epr = \mathop{\mathrm{arginf}}_\varepsilon \left\{ ||\varepsilon||_\infty : f(x_0+\varepsilon) \neq f(x_0) \right\},
    R(f,x_0) := \mathop{\mathrm{inf}}_{\tilde{x}} \left\{ ||\tilde{x}-x_0||_\infty : f(\tilde{x}) \neq y \right\},
\end{equation}
where $x_0\in\Md$ \citep{zhang2021towards}.
This describes the largest region around $x_0$ where $f$ is robust to adversarial attacks. We define the \textit{expected robust radius} of $f$ as $\hat{R}(f) := \E_{x_0\sim\Md}\displaystyle R(f,x_0)$.

\begin{figure}[!htbp]
\centering
\includegraphics[width=0.7\textwidth]{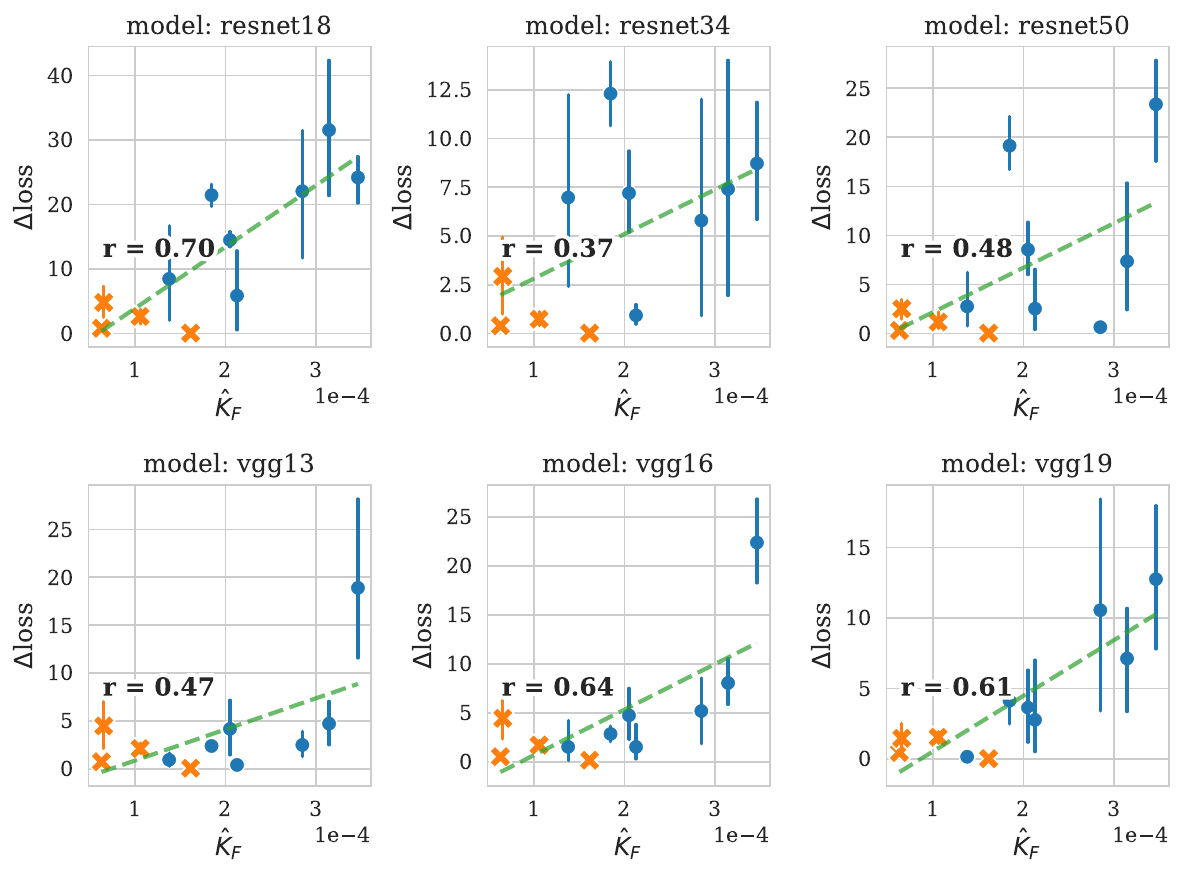}
    \caption{Test set loss penalty due to FGSM adversarial attack vs. measured dataset label sharpness ($\KFh$) for models trained on natural and medical image datasets (orange and blue points, respectively). Pearson correlation coefficient $r$ also shown. Error bars are $95\%$ confidence intervals over all training set sizes $N$ for the same dataset.}
    \label{fig:atk_vs_kf}
\end{figure}

\begin{theorem}[Adversarial Robustness and Label Sharpness Scaling Law]
  \label{thm:advscaling_KF}
  Let $f$ be $K_f$-Lipschitz on $\R^n$. For a sufficiently large training set, the lower bound for the expected robustness radius of $f$ scales as $\hat{R}(f)\simeq\Omega\left({1}/{\KF}\right)$.

\end{theorem}
\begin{proof}
This follows from Prop. 1 of \citet{tsuzuku2018lipschitz} --- see Appendix A.4 for all details.
\end{proof}

While it is very difficult to estimate robustness radii of neural networks in practice \citep{katz2017reluplex}, we can instead measure the average 
%accuracy penalty of $f$ due to attack, $\E_{x_0\sim\Dts}(\mathrm{acc}(\tilde{x}) - \mathrm{acc}(x_0))$,
loss penalty of $f$ due to attack, $\E_{x_0\sim\Dts}(L(\tilde{x}) - L(x_0))$,
over a test set $\Dts$ of points sampled from $\Md$, and see if it correlates negatively with $\KFh$ (Eq. \eqref{eq:KFest}) for different models and datasets. As the expected robustness radius decreases, so should the loss penalty become steeper. We use FGSM \citep{DBLP:journals/corr/GoodfellowSS14} attacks with $L_\infty$ budgets of $\epsilon\in \{1/255, 2/255, 4/225, 8/255\}$ to obtain $\tilde{x}$.

% \begin{wrapfigure}[12]{l}{0.7\textwidth}

In Fig. \ref{fig:atk_vs_kf} we plot the test loss penalty with respect to $\KFh$ for all models and training set sizes for $\epsilon=2/255$, and show the Pearson correlation $r$ between these quantities for each model, for all $\epsilon$, in Table \ref{tab:adv} (per-domain correlations are provided in Appendix E.3). (We provide the plots for the other $\epsilon$ values, as well as for the test \textit{accuracy} penalty, in Appendix E.3). Here we average results over the different training set sizes $N$ due to the lack of dependence of Theorem \ref{thm:advscaling_KF} on $N$.

\begin{wraptable}[10]{r}{0.6\textwidth}
\fontsize{9pt}{9pt}\selectfont
\centering
\begin{tabular}{l||ccc|ccc}
\hline
    Atk. $\epsilon$ & RN-18 & RN-34 & RN-50 & V-13 & V-16 & V-19 \\
    \hline
    $1/255$ & $0.77$ & $0.48$ & $0.55$ & $0.47$ & $0.63$ & $0.61$ \\
    $2/255$ & $0.70$ & $0.37$ & $0.48$ & $0.47$ & $0.64$ & $0.61$ \\
    $4/255$ & $0.63$ & $0.26$ & $0.41$ & $0.45$ & $0.62$ & $0.6$ \\
    $8/255$ & $0.54$ & $0.18$ & $0.34$ & $0.39$ & $0.58$ & $0.57$ \\
    \hline
    \end{tabular}
\caption{Pearson correlation $r$ between test loss penalty due to FGSM attack and dataset label sharpness $\KFh$, over all datasets and all training sizes. ``RN'' = ResNet, ``V'' = VGG.}
\label{tab:adv}
\end{wraptable}

As expected, the loss penalty is typically worse for models trained on datasets with higher $\KF$, implying a smaller expected robustness radius.
We see that medical datasets, which typically have higher $\KF$ than natural datasets (Fig. \ref{fig:kf_vs_dd}), are indeed typically more susceptible to attack, as was found in \citet{ma2021understanding}. In Appendix D.1 we show example clean and attacked images for each medical image dataset for $\epsilon=2/255$. A clinical practitioner may not notice any difference between the clean and attacked images upon first look,\footnote{That being said, the precise physical interpretation of intensity values in certain medical imaging modalities, such as Hounsfield units for CT, may reveal the attack upon close inspection.} yet the attack makes model predictions completely unreliable. This indicates that adversarially-robust models may be needed for medical image analysis scenarios where potential attacks may be a concern.%, due to the safety-critical nature of such tasks.

\section{Connecting Representation Intrinsic Dimension to Dataset Intrinsic Dimension and Generalization}
\label{sec:exp:drepr_scaling}
The scaling of network generalization ability with dataset intrinsic dimension $\dd$ (Sec. \ref{sec:exp:ddata_scaling}) motivates us to study the same behavior in the space of the network's learned hidden representations for the dataset. In particular, we follow \citep{ansuini2019intrinsic,gong2019intrinsic} and assume 
%Related to our study of the intrinsic \textit{dataset} manifold (Sec. \ref{sec:dest}) is the assumption 
that an encoder in a neural network maps input images to some $\dr$-dimensional manifold of \textit{representations} (for a given layer), with $\dr \ll n$. %In this section, we wish to extend our generalization scaling formalism to $\dr$, and connect $\dr$ to $\dd$.
% Conceptually, it is reasonable that the intrinsic manifold dimension of a trained model's encoded representations of a dataset should be correlated with, or bounded by, the intrinsic dimension of the dataset itself, as neural networks are trained to encode their input data into a relatively small number of abstract visual features that can well-describe the data, at least enough to be used for some downstream task. Here, we will provide a more rigorous foundation for this claim. 
As in the empirical work of \citet{ansuini2019intrinsic}, we consider the intrinsic dimensionality of the representations of the final hidden layer of $f$.
Recall that the test loss can be bounded above as $L = \mathcal{O} (K_L\max (K_f, \KF)N^{-1/\dd})$ (Thm. \ref{thm:scaling_dd}). A similar analysis can be used to derive a loss scaling law for $\dr$, as follows.

% Similarly, if we take the assumption that $f$ and $\mathcal{F}$ map the data do some $\dr$-dimensional manifold \citep{ansuini2019intrinsic}, it follows that the loss can be bounded by
% \begin{equation}
% \label{eq:upperbound_dr}
% L\leq K_L \E\limits_{x\in\Dtr}|h(g(x))-\mathcal{H}(g(x))| = \mathcal{O} (K_LD^{-1/\dr}),
% \end{equation}

% \begin{theorem}[Generalization Error and Representation Intrinsic Dim. Scaling Law]
%   \label{thm:scaling_dr}
%   Write $f=h\circ g$ and $\mathcal{F}=\mathcal{H}\circ\mathcal{G}$, where $h$ is the standard sigmoid function. Let $h$ and $\mathcal{H}$ be Lipschitz, and let $g(x)=\mathcal{G}(x)$ and $h(g(x))=\mathcal{H}(g(x))\, \forall x \in \Dtr$. Then, $L\simeq\mathcal{O}(K_LN^{-1/\dr})$, where $K_L$ is the Lipschitz constant for $L$.
% \end{theorem}

\begin{theorem}[Generalization Error and Learned Representation Intrinsic Dimension Scaling Law]
  \label{thm:scaling_dr}
  $L\simeq\mathcal{O}(K_LN^{-1/\dr})$, where $K_L$ is the Lipschitz constant for $L$.
\end{theorem}

We reserve the proof for Appendix A.3 due to length constraints, but the key is to split $f$ into a composition of an encoder and a final layer and analyze the test loss in terms of the encoder's outputted representations.
Similarly to Eq. \eqref{eq:loss_scaling_dd}, $K_L$ is fixed for all experiments, such that we can simplify this result to $ L\simeq\mathcal{O}(N^{-1/\dr})$, \ie, 
\begin{equation}
    \label{eq:dr_logscaling}
    \log L \lesssim -\frac{1}{\dr}\log N + b
\end{equation}
for some constant $b$. This equation is of the same form as the loss scaling law based on the \textit{dataset} intrinsic dimension $\dd$ of Thm. \ref{thm:scaling_dd}. This helps provide theoretical justification for prior empirical results of $L$ increasing with $\dr$ (\cite{ansuini2019intrinsic}, as well as for it being similar in form to the scaling of $L$ with $\dd$ (Fig. \eqref{eq:loss_scaling_dd}).

In Fig. \ref{fig:loss_vs_reprdim} we evaluate the scaling of log test loss with the $\dr$ of the training set (Eq. \eqref{eq:dr_logscaling}), for each model, dataset, and training set size as in Sec. \ref{sec:exp:ddata_scaling}. The estimates of $\dr$ are made using TwoNN on the final hidden layer representations computed from the training set for the given model, as in \citet{ansuini2019intrinsic}. We also show the scaling of test accuracy in Appendix E.1, as well as results from using the MLE estimator to compute $\dr$.

\begin{figure}[!htbp]
\centering
\includegraphics[width=0.98\textwidth]{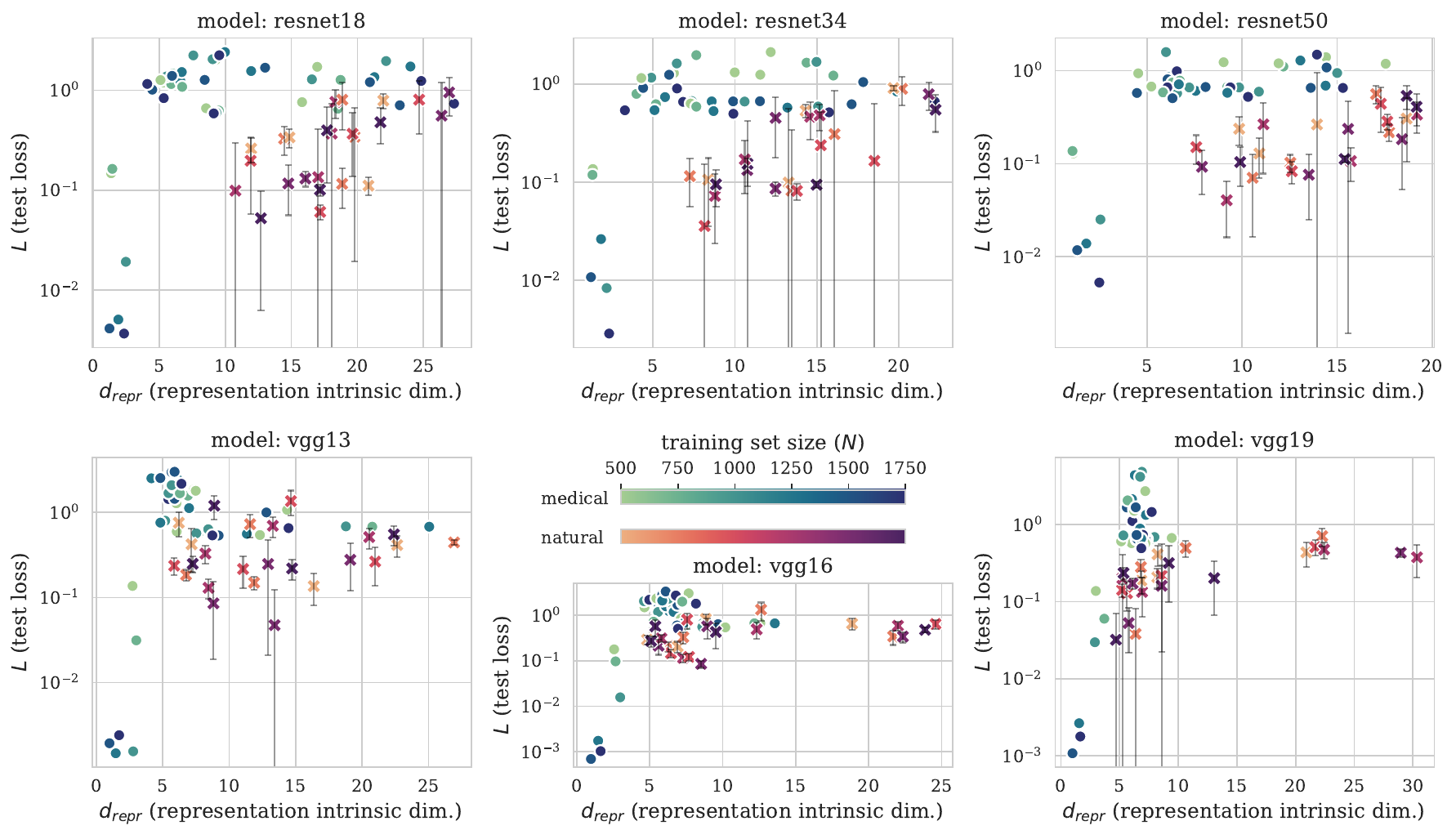}
    \caption{Scaling of log test set loss/generalization ability with the intrinsic dimension of final hidden layer learned representations of the training set ($\dr$), for natural and medical datasets. Each point corresponds to a (model, dataset, training set size) triplet. Medical dataset results are shown in blue shades, and natural dataset results are shown in red.}
    \label{fig:loss_vs_reprdim}
\end{figure}

% In the following section we will use the scaling laws of both $\dd$ and $\dr$ to derive a simple, rough bound for $\dr$ given $\dd$.

% \subsection{Bounding Hidden Representation Intrinsic Dimension with Dataset Intrinsic Dimension}
% \label{sec:ddvsdr}

We see that generalization error typically increases with $\dr$, in a similar shape as the $\dd$ scaling (Fig. \ref{fig:loss_vs_datadim}). The similarity of these curves may be explained by $\dr\lesssim\dd$, or other potential factors unaccounted for. The former arises if the loss bounds of Theorems \ref{thm:scaling_dd} and \ref{thm:scaling_dr} are taken as \textit{estimates}:

\begin{theorem}[Bounding of Representation Intrinsic Dim. with Dataset Intrinsic Dim.]
\label{thm:reprdimbound}
Let Theorems \ref{thm:scaling_dd} and \ref{thm:scaling_dr} be taken as estimates, \ie, $L\approx K_L\max(K_{f},K_{\mathcal{F}})N^{-1/\dd}$ and $L\approx K_LN^{-1/\dr}$. Then, $\dr\lesssim\dd$.
\end{theorem}

\begin{proof}
This centers on equating the two scaling laws and using a property of the Lipschitz constant of classification networks-- see Appendix A.5 for the full proof.
\end{proof}

In other words, the intrinsic dimension of the training dataset serves as an upper bound for the intrinsic dimension of the final hidden layer's learned representations. While a rough estimate, we found this to usually be the case in practice, shown in Fig. \ref{fig:dr_vs_dd} for all models, datasets and training sizes. Here, $\dr=\dd$ is shown as a dashed line, and we use the same estimator (MLE, Sec. \ref{sec:dest}) for $\dd$ and $\dr$ for consistency (similar results using TwoNN are shown in Appendix E.2).

Intuitively, we would expect $\dr$ to be bounded by $\dd$, as $\dd$ encapsulates all raw dataset information, while learned representations prioritize task-related information and discard irrelevant details \citep{tishby2015deep}, resulting in $\dr\lesssim\dd$. Future work could investigate how this relationship varies for networks trained on different tasks, including supervised (\eg, segmentation, detection) and self-supervised or unsupervised learning, where $\dr$ might approach $\dd$.

\section*{Discussion and Conclusions}

In this paper, we explored how the generalization ability and adversarial robustness of a neural network relate to the intrinsic properties of its training set, such as intrinsic dimension ($\dd$) and label sharpness ($\KF$). We chose radiological and natural image domains as prominent examples, but our approach was quite general; indeed, in Appendix C.2 we evaluate our hypotheses on a skin lesion image dataset, a domain that shares similarities with both natural images and radiological images, and intriguingly find that properties of the dataset and models trained on it often lie in between these two domains. It would be interesting to study these relationships in still other imaging domains such as satellite imaging \citep{pritt2017satellite}, histopathology \citep{komura2018machine}, and others. Additionally, this analysis could be extended to other tasks (\eg, multi-class classification or semantic segmentation), newer model architectures such as ConvNeXt \citep{liu2022convnet}, non-convolutional models such as MLPs or vision transformers \citep{dosovitskiy2021an}, or even natural language models.

\begin{wrapfigure}[19]{r}{0.4\textwidth}
    \includegraphics[width=0.4\textwidth]{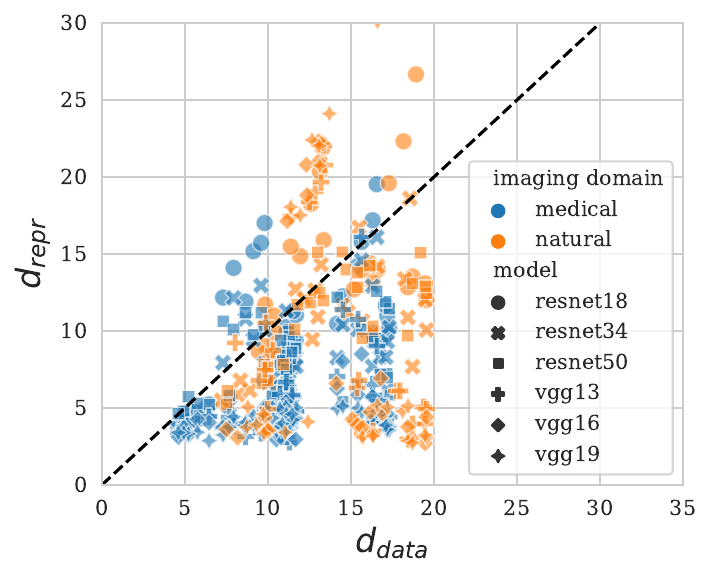}
    \caption{\textbf{Training set intrinsic dimension upper-bounds learned representation intrinsic dimension.} Each point corresponds to a (model, dataset, training set size) triplet.}
    \label{fig:dr_vs_dd}
\end{wrapfigure}

Our findings may provide practical uses beyond merely a better theoretical understanding of these phenomena. For example, we provide a short example of using the network generalization dependence on label sharpness to rank the predicted learning difficulty of different tasks for the same dataset in Appendix C.3. Additionally, the minimum number of annotations needed for an unlabeled training set of images could be inferred given the measured $\dd$ of the dataset and some desired test loss (Eq. \eqref{eq:loss_scaling_dd}), which depends on the imaging domain of the dataset (Fig. \ref{fig:loss_vs_datadim}).\footnote{Note that doing so in practice by fitting the scaling law model to existing ($L$, $N$, $\dd$) results would require first evaluating a wider range of $N$ due to the logarithmic dependence of Eq. \eqref{eq:loss_scaling_dd} on $N$.} This is especially relevant to medical images, where creating quality annotations can be expensive and time-consuming. Additionally, Sec. \ref{sec:exp:robust} demonstrates the importance of using adversarially robust models or training techniques for more vulnerable domains. Finally, the relation of learned representation intrinsic dimension to generalization ability (Sec. \ref{sec:exp:drepr_scaling}) and dataset intrinsic dimension (Theorem \ref{thm:reprdimbound}) could inform the minimum parameter count of network bottleneck layers. 

A limitation of our study is that despite our best efforts, it is difficult to definitively say if training set label sharpness ($\KF$) causes the observed generalization scaling discrepancy between natural and medical image models (Sec. \ref{sec:exp:ddata_scaling}, Fig. \ref{fig:loss_vs_datadim}). We attempted to rule out alternatives via our formal analysis and by constraining many factors in our experiments (\eg, model, loss, training and test set sizes, data sampling strategy, etc.). Additionally, we found that accounting for $\KF$ in the generalization scaling law increases the likelihood of the law given our observed data (Appendix C.1). 
% However, there could always be factors unaccounted for. 
% and the scaling law (Eq. \eqref{eq:loss_scaling_dd}) is just an upper bound. 
Altogether, our results tell us that $\KF$ constitutes an important difference between natural and medical image datasets, but other potential factors unaccounted for should still be considered.

% Generalization scaling patterns vary significantly between medical and natural imaging domains, which may be due to medical datasets exhibiting higher label sharpness, which also leads to increased vulnerability to adversarial attacks. 
% The safety-critical nature of medical image analysis makes it important to develop a rigorous understanding of how neural network behavior in this domain differs from natural images. 
Our findings provide insights into how neural network behavior varies within and between the two crucial domains of natural and medical images, enhancing our understanding of the dependence of generalization ability, representation learning, and adversarial robustness on intrinsic measurable properties of the training set.

\subsubsection*{Author Contributions}
N.K. wrote the paper, derived the mathematical results, ran the experiments, and created the visualizations. M.A.M. helped revise the paper, the presentation of the results, and the key takeaways.

\subsubsection*{Acknowledgments}
The authors would like to thank Hanxue Gu and Haoyu Dong for helpful discussion and inspiration.

% \suggest{Limitations/Future Work}

% \suggest{extend to other domains, such as pathology or fundus?}

% \suggest{we focus on CNNs, it would be interesting to see how this extends to vision transformers, the other leading type of vision arch. An entire other investigation would be into scaling laws for SSL (does something like this exist for vision models? what about language?)}

\bibliography{refs}

\begin{thebibliography}{48}
\providecommand{\natexlab}[1]{#1}
\providecommand{\url}[1]{\texttt{#1}}
\expandafter\ifx\csname urlstyle\endcsname\relax
  \providecommand{\doi}[1]{doi: #1}\else
  \providecommand{\doi}{doi: \begingroup \urlstyle{rm}\Url}\fi

\bibitem[Andreeva et~al.(2023)Andreeva, Limbeck, Rieck, and
  Sarkar]{andreeva2023metric}
Rayna Andreeva, Katharina Limbeck, Bastian Rieck, and Rik Sarkar.
\newblock Metric space magnitude and generalisation in neural networks.
\newblock \emph{arXiv preprint arXiv:2305.05611}, 2023.

\bibitem[Ansuini et~al.(2019)Ansuini, Laio, Macke, and
  Zoccolan]{ansuini2019intrinsic}
Alessio Ansuini, Alessandro Laio, Jakob~H Macke, and Davide Zoccolan.
\newblock Intrinsic dimension of data representations in deep neural networks.
\newblock \emph{Advances in Neural Information Processing Systems}, 32, 2019.

\bibitem[Bahri et~al.(2021)Bahri, Dyer, Kaplan, Lee, and
  Sharma]{bahri2021explaining}
Yasaman Bahri, Ethan Dyer, Jared Kaplan, Jaehoon Lee, and Utkarsh Sharma.
\newblock Explaining neural scaling laws.
\newblock \emph{arXiv preprint arXiv:2102.06701}, 2021.

\bibitem[B{\'e}thune et~al.(2022)B{\'e}thune, Boissin, Serrurier, Mamalet,
  Friedrich, and Gonzalez~Sanz]{bethune2022pay}
Louis B{\'e}thune, Thibaut Boissin, Mathieu Serrurier, Franck Mamalet, Corentin
  Friedrich, and Alberto Gonzalez~Sanz.
\newblock Pay attention to your loss: understanding misconceptions about
  lipschitz neural networks.
\newblock \emph{Advances in Neural Information Processing Systems},
  35:\penalty0 20077--20091, 2022.

\bibitem[Birdal et~al.(2021)Birdal, Lou, Guibas, and
  Simsekli]{birdal2021intrinsic}
Tolga Birdal, Aaron Lou, Leonidas~J Guibas, and Umut Simsekli.
\newblock Intrinsic dimension, persistent homology and generalization in neural
  networks.
\newblock \emph{Advances in Neural Information Processing Systems},
  34:\penalty0 6776--6789, 2021.

\bibitem[Brown et~al.(2023)Brown, Caterini, Ross, Cresswell, and
  Loaiza-Ganem]{brown2023verifying}
Bradley~CA Brown, Anthony~L. Caterini, Brendan~Leigh Ross, Jesse~C Cresswell,
  and Gabriel Loaiza-Ganem.
\newblock Verifying the union of manifolds hypothesis for image data.
\newblock In \emph{The Eleventh International Conference on Learning
  Representations}, 2023.
\newblock URL \url{https://openreview.net/forum?id=Rvee9CAX4fi}.

\bibitem[Caballero et~al.(2023)Caballero, Gupta, Rish, and
  Krueger]{caballero2023broken}
Ethan Caballero, Kshitij Gupta, Irina Rish, and David Krueger.
\newblock Broken neural scaling laws.
\newblock In \emph{The Eleventh International Conference on Learning
  Representations}, 2023.
\newblock URL \url{https://openreview.net/forum?id=sckjveqlCZ}.

\bibitem[Codella et~al.(2018)Codella, Gutman, Celebi, Helba, Marchetti, Dusza,
  Kalloo, Liopyris, Mishra, Kittler, et~al.]{codella2018skin}
Noel~CF Codella, David Gutman, M~Emre Celebi, Brian Helba, Michael~A Marchetti,
  Stephen~W Dusza, Aadi Kalloo, Konstantinos Liopyris, Nabin Mishra, Harald
  Kittler, et~al.
\newblock Skin lesion analysis toward melanoma detection: A challenge at the
  2017 international symposium on biomedical imaging (isbi), hosted by the
  international skin imaging collaboration (isic).
\newblock In \emph{2018 IEEE 15th international symposium on biomedical imaging
  (ISBI 2018)}, pp.\  168--172. IEEE, 2018.

\bibitem[Deng et~al.(2009)Deng, Dong, Socher, Li, Li, and
  Fei-Fei]{deng2009imagenet}
Jia Deng, Wei Dong, Richard Socher, Li-Jia Li, Kai Li, and Li~Fei-Fei.
\newblock Imagenet: A large-scale hierarchical image database.
\newblock In \emph{2009 IEEE conference on computer vision and pattern
  recognition}, pp.\  248--255. Ieee, 2009.

\bibitem[Deng(2012)]{deng2012mnist}
Li~Deng.
\newblock The mnist database of handwritten digit images for machine learning
  research [best of the web].
\newblock \emph{IEEE signal processing magazine}, 29\penalty0 (6):\penalty0
  141--142, 2012.

\bibitem[Dosovitskiy et~al.(2021)Dosovitskiy, Beyer, Kolesnikov, Weissenborn,
  Zhai, Unterthiner, Dehghani, Minderer, Heigold, Gelly, Uszkoreit, and
  Houlsby]{dosovitskiy2021an}
Alexey Dosovitskiy, Lucas Beyer, Alexander Kolesnikov, Dirk Weissenborn,
  Xiaohua Zhai, Thomas Unterthiner, Mostafa Dehghani, Matthias Minderer, Georg
  Heigold, Sylvain Gelly, Jakob Uszkoreit, and Neil Houlsby.
\newblock An image is worth 16x16 words: Transformers for image recognition at
  scale.
\newblock In \emph{International Conference on Learning Representations}, 2021.
\newblock URL \url{https://openreview.net/forum?id=YicbFdNTTy}.

\bibitem[Facco et~al.(2017)Facco, d’Errico, Rodriguez, and
  Laio]{facco2017estimating}
Elena Facco, Maria d’Errico, Alex Rodriguez, and Alessandro Laio.
\newblock Estimating the intrinsic dimension of datasets by a minimal
  neighborhood information.
\newblock \emph{Scientific reports}, 7\penalty0 (1):\penalty0 12140, 2017.

\bibitem[Fazlyab et~al.(2019)Fazlyab, Robey, Hassani, Morari, and
  Pappas]{fazlyab2019efficient}
Mahyar Fazlyab, Alexander Robey, Hamed Hassani, Manfred Morari, and George
  Pappas.
\newblock Efficient and accurate estimation of lipschitz constants for deep
  neural networks.
\newblock \emph{Advances in Neural Information Processing Systems}, 32, 2019.

\bibitem[Fefferman et~al.(2016)Fefferman, Mitter, and
  Narayanan]{fefferman2016manifoldhyp}
Charles Fefferman, Sanjoy Mitter, and Hariharan Narayanan.
\newblock Testing the manifold hypothesis.
\newblock \emph{Journal of the American Mathematical Society}, 29\penalty0
  (4):\penalty0 983--1049, 2016.

\bibitem[Flanders et~al.(2020)Flanders, Prevedello, Shih, Halabi,
  Kalpathy-Cramer, Ball, Mongan, Stein, Kitamura, Lungren,
  et~al.]{flanders2020rsnaihct}
Adam~E Flanders, Luciano~M Prevedello, George Shih, Safwan~S Halabi, Jayashree
  Kalpathy-Cramer, Robyn Ball, John~T Mongan, Anouk Stein, Felipe~C Kitamura,
  Matthew~P Lungren, et~al.
\newblock Construction of a machine learning dataset through collaboration: the
  rsna 2019 brain ct hemorrhage challenge.
\newblock \emph{Radiology: Artificial Intelligence}, 2\penalty0 (3):\penalty0
  e190211, 2020.

\bibitem[Gao \& Pavel(2017)Gao and Pavel]{gao2017properties}
Bolin Gao and Lacra Pavel.
\newblock On the properties of the softmax function with application in game
  theory and reinforcement learning.
\newblock \emph{arXiv preprint arXiv:1704.00805}, 2017.

\bibitem[Gong et~al.(2019)Gong, Boddeti, and Jain]{gong2019intrinsic}
Sixue Gong, Vishnu~Naresh Boddeti, and Anil~K Jain.
\newblock On the intrinsic dimensionality of image representations.
\newblock In \emph{Proceedings of the IEEE/CVF Conference on Computer Vision
  and Pattern Recognition}, pp.\  3987--3996, 2019.

\bibitem[Goodfellow et~al.(2015)Goodfellow, Shlens, and
  Szegedy]{DBLP:journals/corr/GoodfellowSS14}
Ian~J. Goodfellow, Jonathon Shlens, and Christian Szegedy.
\newblock Explaining and harnessing adversarial examples.
\newblock In Yoshua Bengio and Yann LeCun (eds.), \emph{3rd International
  Conference on Learning Representations, {ICLR} 2015, San Diego, CA, USA, May
  7-9, 2015, Conference Track Proceedings}, 2015.
\newblock URL \url{http://arxiv.org/abs/1412.6572}.

\bibitem[He et~al.(2016)He, Zhang, Ren, and Sun]{he2016resnet}
Kaiming He, Xiangyu Zhang, Shaoqing Ren, and Jian Sun.
\newblock Deep residual learning for image recognition.
\newblock In \emph{Proceedings of the IEEE conference on computer vision and
  pattern recognition}, pp.\  770--778, 2016.

\bibitem[Hoffmann et~al.(2022)Hoffmann, Borgeaud, Mensch, Buchatskaya, Cai,
  Rutherford, Casas, Hendricks, Welbl, Clark, et~al.]{hoffmann2022training}
Jordan Hoffmann, Sebastian Borgeaud, Arthur Mensch, Elena Buchatskaya, Trevor
  Cai, Eliza Rutherford, Diego de~Las Casas, Lisa~Anne Hendricks, Johannes
  Welbl, Aidan Clark, et~al.
\newblock Training compute-optimal large language models.
\newblock \emph{arXiv preprint arXiv:2203.15556}, 2022.

\bibitem[Irvin et~al.(2019)Irvin, Rajpurkar, Ko, Yu, Ciurea-Ilcus, Chute,
  Marklund, Haghgoo, Ball, Shpanskaya, et~al.]{irvin2019chexpert}
Jeremy Irvin, Pranav Rajpurkar, Michael Ko, Yifan Yu, Silviana Ciurea-Ilcus,
  Chris Chute, Henrik Marklund, Behzad Haghgoo, Robyn Ball, Katie Shpanskaya,
  et~al.
\newblock Chexpert: A large chest radiograph dataset with uncertainty labels
  and expert comparison.
\newblock In \emph{Proceedings of the AAAI conference on artificial
  intelligence}, volume~33, pp.\  590--597, 2019.

\bibitem[Kaplan et~al.(2020)Kaplan, McCandlish, Henighan, Brown, Chess, Child,
  Gray, Radford, Wu, and Amodei]{kaplan2020scaling}
Jared Kaplan, Sam McCandlish, Tom Henighan, Tom~B Brown, Benjamin Chess, Rewon
  Child, Scott Gray, Alec Radford, Jeffrey Wu, and Dario Amodei.
\newblock Scaling laws for neural language models.
\newblock \emph{arXiv preprint arXiv:2001.08361}, 2020.

\bibitem[Katz et~al.(2017)Katz, Barrett, Dill, Julian, and
  Kochenderfer]{katz2017reluplex}
Guy Katz, Clark Barrett, David~L Dill, Kyle Julian, and Mykel~J Kochenderfer.
\newblock Reluplex: An efficient smt solver for verifying deep neural networks.
\newblock In \emph{Computer Aided Verification: 29th International Conference,
  CAV 2017, Heidelberg, Germany, July 24-28, 2017, Proceedings, Part I 30},
  pp.\  97--117. Springer, 2017.

\bibitem[Kingma \& Ba(2015)Kingma and Ba]{adam}
Diederik~P. Kingma and Jimmy Ba.
\newblock Adam: {A} method for stochastic optimization.
\newblock In Yoshua Bengio and Yann LeCun (eds.), \emph{3rd International
  Conference on Learning Representations, {ICLR} 2015, San Diego, CA, USA, May
  7-9, 2015, Conference Track Proceedings}, 2015.
\newblock URL \url{http://arxiv.org/abs/1412.6980}.

\bibitem[Komura \& Ishikawa(2018)Komura and Ishikawa]{komura2018machine}
Daisuke Komura and Shumpei Ishikawa.
\newblock Machine learning methods for histopathological image analysis.
\newblock \emph{Computational and structural biotechnology journal},
  16:\penalty0 34--42, 2018.

\bibitem[Konz et~al.(2022)Konz, Gu, Dong, and Mazurowski]{konz2022intrinsic}
Nicholas Konz, Hanxue Gu, Haoyu Dong, and Maciej~A Mazurowski.
\newblock The intrinsic manifolds of radiological images and their role in deep
  learning.
\newblock In \emph{Medical Image Computing and Computer Assisted
  Intervention--MICCAI 2022: 25th International Conference, Singapore,
  September 18--22, 2022, Proceedings, Part VIII}, pp.\  684--694. Springer,
  2022.

\bibitem[Krizhevsky et~al.(2009)Krizhevsky, Hinton,
  et~al.]{krizhevsky2009cifar}
Alex Krizhevsky, Geoffrey Hinton, et~al.
\newblock Learning multiple layers of features from tiny images.
\newblock 2009.

\bibitem[Kvinge et~al.(2023)Kvinge, Brown, and Godfrey]{kvinge2023exploring}
Henry Kvinge, Davis Brown, and Charles Godfrey.
\newblock Exploring the representation manifolds of stable diffusion through
  the lens of intrinsic dimension.
\newblock \emph{ICLR 2023 Workshop on Mathematical and Empirical Understanding
  of Foundation Models}, 2023.

\bibitem[Levina \& Bickel(2004)Levina and Bickel]{levina2004maximum}
Elizaveta Levina and Peter Bickel.
\newblock Maximum likelihood estimation of intrinsic dimension.
\newblock \emph{Advances in neural information processing systems}, 17, 2004.

\bibitem[Liu et~al.(2022)Liu, Mao, Wu, Feichtenhofer, Darrell, and
  Xie]{liu2022convnet}
Zhuang Liu, Hanzi Mao, Chao-Yuan Wu, Christoph Feichtenhofer, Trevor Darrell,
  and Saining Xie.
\newblock A convnet for the 2020s.
\newblock In \emph{Proceedings of the IEEE/CVF conference on computer vision
  and pattern recognition}, pp.\  11976--11986, 2022.

\bibitem[Ma et~al.(2021)Ma, Niu, Gu, Wang, Zhao, Bailey, and
  Lu]{ma2021understanding}
Xingjun Ma, Yuhao Niu, Lin Gu, Yisen Wang, Yitian Zhao, James Bailey, and Feng
  Lu.
\newblock Understanding adversarial attacks on deep learning based medical
  image analysis systems.
\newblock \emph{Pattern Recognition}, 110:\penalty0 107332, 2021.

\bibitem[MacKay \& Ghahramani(2005)MacKay and Ghahramani]{mackay2005comments}
David~JC MacKay and Zoubin Ghahramani.
\newblock Comments on'maximum likelihood estimation of intrinsic dimension'by
  e. levina and p. bickel (2004).
\newblock \emph{The Inference Group Website, Cavendish Laboratory, Cambridge
  University}, 2005.

\bibitem[Menze et~al.(2014)Menze, Jakab, Bauer, Kalpathy-Cramer, Farahani,
  Kirby, Burren, Porz, Slotboom, Wiest, et~al.]{menze2014multimodal}
Bjoern~H Menze, Andras Jakab, Stefan Bauer, Jayashree Kalpathy-Cramer, Keyvan
  Farahani, Justin Kirby, Yuliya Burren, Nicole Porz, Johannes Slotboom, Roland
  Wiest, et~al.
\newblock The multimodal brain tumor image segmentation benchmark (brats).
\newblock \emph{IEEE transactions on medical imaging}, 34\penalty0
  (10):\penalty0 1993--2024, 2014.

\bibitem[Netzer et~al.(2011)Netzer, Wang, Coates, Bissacco, Wu, and
  Ng]{netzer2011svhn}
Yuval Netzer, Tao Wang, Adam Coates, Alessandro Bissacco, Bo~Wu, and Andrew~Y
  Ng.
\newblock Reading digits in natural images with unsupervised feature learning.
\newblock 2011.

\bibitem[Pope et~al.(2020)Pope, Zhu, Abdelkader, Goldblum, and
  Goldstein]{pope2021intrinsic}
Phil Pope, Chen Zhu, Ahmed Abdelkader, Micah Goldblum, and Tom Goldstein.
\newblock The intrinsic dimension of images and its impact on learning.
\newblock In \emph{International Conference on Learning Representations}, 2020.

\bibitem[Pritt \& Chern(2017)Pritt and Chern]{pritt2017satellite}
Mark Pritt and Gary Chern.
\newblock Satellite image classification with deep learning.
\newblock In \emph{2017 IEEE applied imagery pattern recognition workshop
  (AIPR)}, pp.\  1--7. IEEE, 2017.

\bibitem[Rajpurkar et~al.(2017)Rajpurkar, Irvin, Bagul, Ding, Duan, Mehta,
  Yang, Zhu, Laird, Ball, et~al.]{rajpurkar2017mura}
Pranav Rajpurkar, Jeremy Irvin, Aarti Bagul, Daisy Ding, Tony Duan, Hershel
  Mehta, Brandon Yang, Kaylie Zhu, Dillon Laird, Robyn~L Ball, et~al.
\newblock Mura: Large dataset for abnormality detection in musculoskeletal
  radiographs.
\newblock \emph{arXiv preprint arXiv:1712.06957}, 2017.

\bibitem[Saha et~al.(2018)Saha, Harowicz, Grimm, Kim, Ghate, Walsh, and
  Mazurowski]{saha2018machinedukedbc}
Ashirbani Saha, Michael~R Harowicz, Lars~J Grimm, Connie~E Kim, Sujata~V Ghate,
  Ruth Walsh, and Maciej~A Mazurowski.
\newblock A machine learning approach to radiogenomics of breast cancer: a
  study of 922 subjects and 529 dce-mri features.
\newblock \emph{British journal of cancer}, 119\penalty0 (4):\penalty0
  508--516, 2018.

\bibitem[Simonyan \& Zisserman(2015)Simonyan and Zisserman]{VGGs}
Karen Simonyan and Andrew Zisserman.
\newblock Very deep convolutional networks for large-scale image recognition.
\newblock In Yoshua Bengio and Yann LeCun (eds.), \emph{3rd International
  Conference on Learning Representations, {ICLR} 2015, San Diego, CA, USA, May
  7-9, 2015, Conference Track Proceedings}, 2015.
\newblock URL \url{http://arxiv.org/abs/1409.1556}.

\bibitem[Sonn et~al.(2013)Sonn, Natarajan, Margolis, MacAiran, Lieu, Huang,
  Dorey, and Marks]{sonn2013prostate}
Geoffrey~A Sonn, Shyam Natarajan, Daniel~JA Margolis, Malu MacAiran, Patricia
  Lieu, Jiaoti Huang, Frederick~J Dorey, and Leonard~S Marks.
\newblock Targeted biopsy in the detection of prostate cancer using an office
  based magnetic resonance ultrasound fusion device.
\newblock \emph{The Journal of urology}, 189\penalty0 (1):\penalty0 86--92,
  2013.

\bibitem[Tishby \& Zaslavsky(2015)Tishby and Zaslavsky]{tishby2015deep}
Naftali Tishby and Noga Zaslavsky.
\newblock Deep learning and the information bottleneck principle.
\newblock In \emph{2015 ieee information theory workshop (itw)}, pp.\  1--5.
  IEEE, 2015.

\bibitem[Tiulpin et~al.(2018)Tiulpin, Thevenot, Rahtu, Lehenkari, and
  Saarakkala]{Tiulpin2018}
Aleksei Tiulpin, J{\'{e}}r{\^{o}}me Thevenot, Esa Rahtu, Petri Lehenkari, and
  Simo Saarakkala.
\newblock {Automatic Knee Osteoarthritis Diagnosis from Plain Radiographs: A
  Deep Learning-Based Approach}.
\newblock \emph{Scientific Reports}, 8\penalty0 (1):\penalty0 1727, 2018.
\newblock ISSN 2045-2322.
\newblock \doi{10.1038/s41598-018-20132-7}.
\newblock URL \url{https://doi.org/10.1038/s41598-018-20132-7}.

\bibitem[Touvron et~al.(2023)Touvron, Martin, Stone, Albert, Almahairi, Babaei,
  Bashlykov, Batra, Bhargava, Bhosale, et~al.]{touvron2023llama}
Hugo Touvron, Louis Martin, Kevin Stone, Peter Albert, Amjad Almahairi, Yasmine
  Babaei, Nikolay Bashlykov, Soumya Batra, Prajjwal Bhargava, Shruti Bhosale,
  et~al.
\newblock Llama 2: Open foundation and fine-tuned chat models.
\newblock \emph{arXiv preprint arXiv:2307.09288}, 2023.

\bibitem[Tsuzuku et~al.(2018)Tsuzuku, Sato, and Sugiyama]{tsuzuku2018lipschitz}
Yusuke Tsuzuku, Issei Sato, and Masashi Sugiyama.
\newblock Lipschitz-margin training: Scalable certification of perturbation
  invariance for deep neural networks.
\newblock \emph{Advances in neural information processing systems}, 31, 2018.

\bibitem[Virtanen et~al.(2020)Virtanen, Gommers, Oliphant, Haberland, Reddy,
  Cournapeau, Burovski, Peterson, Weckesser, Bright, {van der Walt}, Brett,
  Wilson, Millman, Mayorov, Nelson, Jones, Kern, Larson, Carey, Polat, Feng,
  Moore, {VanderPlas}, Laxalde, Perktold, Cimrman, Henriksen, Quintero, Harris,
  Archibald, Ribeiro, Pedregosa, {van Mulbregt}, and {SciPy 1.0
  Contributors}]{2020SciPy-NMeth}
Pauli Virtanen, Ralf Gommers, Travis~E. Oliphant, Matt Haberland, Tyler Reddy,
  David Cournapeau, Evgeni Burovski, Pearu Peterson, Warren Weckesser, Jonathan
  Bright, St{\'e}fan~J. {van der Walt}, Matthew Brett, Joshua Wilson, K.~Jarrod
  Millman, Nikolay Mayorov, Andrew R.~J. Nelson, Eric Jones, Robert Kern, Eric
  Larson, C~J Carey, {\.I}lhan Polat, Yu~Feng, Eric~W. Moore, Jake
  {VanderPlas}, Denis Laxalde, Josef Perktold, Robert Cimrman, Ian Henriksen,
  E.~A. Quintero, Charles~R. Harris, Anne~M. Archibald, Ant{\^o}nio~H. Ribeiro,
  Fabian Pedregosa, Paul {van Mulbregt}, and {SciPy 1.0 Contributors}.
\newblock {{SciPy} 1.0: Fundamental Algorithms for Scientific Computing in
  Python}.
\newblock \emph{Nature Methods}, 17:\penalty0 261--272, 2020.
\newblock \doi{10.1038/s41592-019-0686-2}.

\bibitem[Vuong(1989)]{vuong1989likelihood}
Quang~H Vuong.
\newblock Likelihood ratio tests for model selection and non-nested hypotheses.
\newblock \emph{Econometrica: journal of the Econometric Society}, pp.\
  307--333, 1989.

\bibitem[Yang et~al.(2023)Yang, Shi, Wei, Liu, Zhao, Ke, Pfister, and
  Ni]{medmnistv2}
Jiancheng Yang, Rui Shi, Donglai Wei, Zequan Liu, Lin Zhao, Bilian Ke,
  Hanspeter Pfister, and Bingbing Ni.
\newblock Medmnist v2-a large-scale lightweight benchmark for 2d and 3d
  biomedical image classification.
\newblock \emph{Scientific Data}, 10\penalty0 (1):\penalty0 41, 2023.

\bibitem[Zhang et~al.(2021)Zhang, Cai, Lu, He, and Wang]{zhang2021towards}
Bohang Zhang, Tianle Cai, Zhou Lu, Di~He, and Liwei Wang.
\newblock Towards certifying l-infinity robustness using neural networks with
  l-inf-dist neurons.
\newblock In \emph{International Conference on Machine Learning}, pp.\
  12368--12379. PMLR, 2021.

\end{thebibliography}
\bibliographystyle{iclr2024_conference}

\clearpage
\appendix

\doparttoc % Tell to minitoc to generate a toc for the parts
\faketableofcontents % Run a fake tableofcontents command for the partocs

\addcontentsline{toc}{section}{Appendix} % Add the appendix text to the document TOC
\part{Supplementary Materials} 
\parttoc

\section{Mathematical Details and Proofs}
\subsection{Extension of Results to Multi-Class Classification}
\label{app:multiclass}

\paragraph{Generalization scaling laws.}
Our results extend naturally from binary classification to multi-class classification. Given some test point $x_0$ of some unknown target class,
if $x'_{tr}$ is the nearest neighbor to $x_0$ in the training set of the same class (both on $\Md$), the term $\mathop{\E}_{\Dtr\sim\Md}||x_0-x'_{tr}||$ scales in expectation as 
\begin{equation}
\mathcal{O}\left(\left(\frac{N+1}{C}\right)^{-1/\dd}\right) \simeq \mathcal{O}\left(\left(\frac{N}{C}\right)^{-1/\dd}\right)=\mathcal{O}(N^{-1/\dd}),
\end{equation}
where $C$ is the total number of classes, assuming the classes to be evenly sampled in the training set. The same logic can be used for the intrinsic representation dimension $\dr$ to show $\mathcal{O}\left(\left(\frac{N+1}{C}\right)^{-1/\dr}\right) \simeq \mathcal{O}(N^{-1/\dr})$. Therefore, the asymptotic upper bounds in the $\dd$ and $\dr$ scaling laws (Theorems \ref{thm:scaling_dd} and \ref{thm:scaling_dr}, respectively) still hold, as well as the derived result of Theorem \ref{thm:reprdimbound}.

\paragraph{Label sharpness.}
The label sharpness metric $\KFh$ (Eq. \ref{eq:KFest}) was formulated under the binary classification scenario, where data is either labeled with $0$ or $1$ (Sec. \ref{sec:preliminaries}). However, it could potentially be extended to the multi-class scenario by simply replacing the $|y_j-y_k|$ term in the numerator of Eq. \ref{eq:KFest} with the indicator function $1_{y_j\neq y_k}$ as
\begin{equation}
     \KFh := \max\limits_{j, k} \left(\frac{1_{y_j\neq y_k}}{|| x_j - x_k||}\right),
\end{equation}
which clearly simplifies to Eq. \ref{eq:KFest} for binary classification. This modification prevents $\KFh$ from being biased by the numerical value of labels given to different classes, but a more careful extension could be pursued in the future to confirm a properly theoretically-motivated multi-class label sharpness metric.

\subsection{Proof of Theorem \ref{thm:KfsimKF} (Approximating $K_f$ with $K_\mathcal{F}$)}
\label{app:proof:KfsimKF}

% \suggest{or should it be that the estimates for these converge to each other? Better wording might be ``$K_f\simeq \KF$ for large $N$.''}
\begin{proof}
Let $x_1$ and $x_2$ be arbitrary datapoints sampled from $\Md$, with nearest neighbors in the training set $\Dtr$ of $\xh_1$ and $\xh_2$, respectively. Then,
\begin{multline}
  |f(x_1)-f(x_2)| = |f(x_1)-f(x_2) + (\mathcal{F}(x_1) - \mathcal{F}(x_1) + \mathcal{F}(x_2) - \mathcal{F}(x_2) \\
  + f(\xh_1) - f(\xh_1) + f(\xh_2) - f(\xh_2))| \\
  \leq |f(x_1) - f(\xh_1)| + |f(x_2) - f(\xh_2)| + |\mathcal{F}(x_1) - \mathcal{F}(x_2)|  \\
   + | f(\xh_1) - \mathcal{F}(x_1)| + |f(\xh_2) - \mathcal{F}(x_2) |,
\end{multline}
by the triangle inequality. Because we assumed that $f(x)=\mathcal{F}(x)\,\forall x\in\Dtr$, \ie, the model is well-trained, the last two terms can be changed so that we have 
\begin{multline}
  |f(x_1)-f(x_2)| \leq |f(x_1) - f(\xh_1)| + |f(x_2) - f(\xh_2)| + |\mathcal{F}(x_1) - \mathcal{F}(x_2)|  \\
   + | \mathcal{F}(\xh_1) - \mathcal{F}(x_1)| + |\mathcal{F}(\xh_2) - \mathcal{F}(x_2) |.
\end{multline}
Using the Lipschitz continuity of $f$ and $\mathcal{F}$, we have that
\begin{multline}
\label{eq:lipschitzconsteqproofline}
  |f(x_1)-f(x_2)| \leq K_f(||x_1-\xh_1||+||x_2-\xh_2||)+\KF(||x_1-x_2||+||x_1-\xh_1||+||x_2-\xh_2||) \\
  = \KF||x_1-x_2|| + (K_f+\KF)(||x_1-\xh_1||+||x_2-\xh_2||).
\end{multline}
Recall that the expected nearest-neighbor distance on $\Md$ for some $N$ samples scales as $\mathcal{O}(N^{-1/\dd})$. Then, $\E||x_1-\hat{x}_1|| = \E||x_2-\hat{x}_2|| = \mathcal{O}((N+1)^{-1/\dd}) \simeq \mathcal{O}(N^{-1/\dd})$. If we take the expectation of both sides of Eq. \eqref{eq:lipschitzconsteqproofline} over the training set, we can use this fact to obtain
\begin{equation}
  \E|f(x_1)-f(x_2)| \leq \KF\E||x_1-x_2|| + \mathcal{O}(\max(K_f,\KF)(N^{-1/\dd})).
\end{equation}
But, the term on the right goes to zero as $N\rightarrow\infty$, so then $\mathrm{Pr}\left(|f(x_1)-f(x_2)| \leq \KF||x_1-x_2||\right)\rightarrow 1$ as $N\rightarrow\infty$, or in other words, the probability that $f$ is Lipschitz with the same constant $\KF$ of $\mathcal{F}$. (A very similar proof can also be made to show that $\mathrm{Pr}\left(|\mathcal{F}(x_1)-\mathcal{F}(x_2)| \leq K_f||x_1-x_2||\right)\rightarrow 1$ as $N\rightarrow\infty$). Therefore, the Lipschitz constant of $f$ converges to $\KF$ in probability, or in other words, $K_f\rightarrow \KF$.
% \end{align}
\end{proof}

\subsection{Proof of Theorem \ref{thm:scaling_dr} (Generalization Error and Representation Intrinsic Dim. Scaling Law)}
\label{app:proof:drscaling}

\begin{proof}
Let $f$ be written as a composition of an encoder $g$, which outputs the final hidden representations of the input image, and a final output sigmoid (or softmax for multi-class classification) layer $h$, as $f=h\circ g$. Write the true label function $\mathcal{F}$ similarly, as some $\mathcal{F}=\mathcal{H}\circ \mathcal{G}$ for unknown $\mathcal{H}$ and $\mathcal{G}$ analogous to $h$ and $g$.
% Next, we will follow the same prescription as Theorem \ref{thm:scaling_dd} \citep{bahri2021explaining}. 
Assume $h$ and $\mathcal{H}$ to be Lipschitz with respective constants $K_{h}$ and $K_{\mathcal{H}}$. 
Analogous to assuming $f(x)=\mathcal{F}(x)$ for all $x$ in the training set $\Dtr$, posit a similar claim of $g(x)=\mathcal{G}(x):= z$, and $h(z)=\mathcal{H}(z)$, $\forall x\in\Dtr$.

Let $x$ be from the training set $\Dtr$ with nearest neighbor (also in the training set) $\xh$. Recall that we assume that $g(x)=\mathcal{G}(x)\,\forall x\in\Dtr$, and that the loss vanishes at the true target label, as in \citet{bahri2021explaining}. Let $z=g(x)$ and $\zz=g(\xh)$.

Then, as we assumed $f$ and $\mathcal{F}$ to be Lipschitz, 
\begin{align}
  \ell(f(x)) = |\ell(f(x)) - \ell(\mathcal{F}(x)) | &\leq K_L|f(x)-\mathcal{F}(x)| \\
  & = K_L|h(g(x)) - \mathcal{H}(\mathcal{G}(x))| = K_L|h(z) - \mathcal{H}(z)|
\end{align}
where $\ell(f(x))$ is the loss evaluated at a single datapoint, and the first equality is due to the loss vanishing at the true target label ($ \ell(\mathcal{F}(x))=0$), and being non-negative. Continuing,
\begin{align}
    \ell(f(x))  &\leq  K_L|h(z) - \mathcal{H}(z)|\\ 
                &= K_L| h(z) - \mathcal{H}(z) + (h(\zz) - h(\zz) + \mathcal{H}(\zz) - \mathcal{H}(\zz))| \\
                &\leq K_L\left(|h(z)-h(\zz)| + |\mathcal{H}(z)-\mathcal{H}(\zz)| + |h(\zz) - \mathcal{H}(\zz)|\right),
\end{align}
with the last line from the triangle inequality. As $h(z)=\mathcal{H}(z)$ for all $\{z=g(x) :x\in\Dtr\}$, the last term vanishes, allowing us to write
\begin{align}
  \ell(f(x))  &\leq K_L\left(|h(z)-h(\zz)| + |\mathcal{H}(z)-\mathcal{H}(\zz)|\right) \\
  &\leq K_L\left(K_{h}||z-\zz|| + K_{\mathcal{H}}||z-\zz||\right) = K_L(K_{h} + K_{\mathcal{H}})||z-\zz||,
\end{align}
so then 
\begin{align}
  \ell(f(x))  &\leq K_L\left(K_{h}||z-\zz|| + K_{\mathcal{H}}||z-\zz||\right) = K_L(K_{h} + K_{\mathcal{H}})||z-\zz||,
\end{align}
and
\begin{align}
  L &= \E\limits_{x\sim\Dts}\ell(f(x)) \leq  K_L(K_{h} + K_{\mathcal{H}})\E\limits_{z,\zz\sim\Dtr}||z-\zz||,
\end{align}
where the rightmost expectation is taken over all $\{z=g(x) :x\in\Dtr\}$ with corresponding nearest neighbor $\hat{z}$ (on the representation manifold). As the expectation of the nearest-neighbor distance of the representations on the manifold scales as $\mathcal{O} (N^{-1/\dr})$, it follows that  
\begin{align}
  L &= \mathcal{O} (K_L\max(K_{h},K_{\mathcal{H}})N^{-1/\dr}).
\end{align}
Because $h=\mathcal{H}$ on the training set representations, the same procedure as the proof for Theorem \ref{thm:KfsimKF} can be used to show that $K_{\mathcal{H}}\simeq K_{h}$. Finally, note that the output layer $h$ was assumed to be a sigmoid. As the standard sigmoid (or softmax) layer is $1-$Lipschitz \citep{gao2017properties}, $K_{\mathcal{H}}\simeq K_{h}=1$, so then
\begin{align}
  L &\simeq \mathcal{O} (K_LN^{-1/\dr}).
\end{align}
\end{proof}

\subsection{Proof of Theorem \ref{thm:advscaling_KF} (Adversarial Robustness and Label Sharpness Scaling Law)}
\label{app:proof:kfscaling}

\begin{proof}
Proposition 1 of \citep{tsuzuku2018lipschitz} states that $\hat{R}(f, x_0)\geq M_{f,x_0}/(\sqrt{2}K_f)$ where $M_{f,x_0}>0$ is the prediction margin, the difference between the target class prediction and the highest non-target class prediction of $f(x_0)$. Applying Thm. \ref{thm:KfsimKF} given sufficiently large $N$ then gives $\hat{R}(f) = \E_{x_0\sim\Md} \hat{R}(f, x_0) \geq \E_{x_0\sim\Md} M_{f,x_0}/(\sqrt{2}K_f) = \Omega\left({1}/{K_f}\right) \simeq\Omega\left({1}/{\KF}\right)$.
\end{proof}

\subsection{Proof of Theorem \ref{thm:reprdimbound} (Bounding of Representation Intrinsic Dim. with Dataset Intrinsic Dim.)}
\label{app:proof:reprdimbound}

\begin{proof}
The estimation assumption implies that 
\begin{equation}
 K_L\max(K_{f},K_{\mathcal{F}})N^{-1/\dd}\approx K_LN^{-1/\dr} \quad\Rightarrow\quad N^{-1/\dd} \approx \frac{N^{-1/\dr}}{\max(K_{f},K_{\mathcal{F}})},
\end{equation}
after which taking the logarithm of both sides gives
\begin{equation}
  \label{eq:d_cases}
  \begin{cases}
    \dr \lesssim \dd & \text{if } K_f,\KF\leq 1 \\
    \dr \gtrsim \dd  & \text{otherwise,}
  \end{cases}
\end{equation}
\ie, $\dr \lesssim \dd$ if the trained model $f$ and target model $\mathcal{F}$ are 1-Lipschitz (with respect to nearest neighbors in the training set). 

Now, note that in our classification task setting, the decision boundaries/predictions of some $K$-Lipschitz network $f$ are the same as the $1-$Lipschitz version $\frac{1}{K}f$ \citep{bethune2022pay}. As such, the scaling behavior we analyze here of $L$ vs. $\dr$ is the same as if $K_f=1$. As $\KF\simeq K_f$ (Theorem \ref{thm:KfsimKF}), Eq. \eqref{eq:d_cases} can be simplified to just $\dr \lesssim \dd$. In practice, we also found that all datasets had $\KF \ll 1$ (Fig. \ref{fig:kf_vs_dd}), so the first case of Eq. \eqref{eq:d_cases} should hold true anyways.

% As such, we can do this rescaling of $f\leftarrow\frac{1}{K_f}f$ and $\mathcal{F}\leftarrow\frac{1}{\KF}\mathcal{F}$ without affecting the loss behavior, allowing $f$ and $\mathcal{F}$ to be taken as $1-$Lipschitz. Then, $\dr \lesssim \dd$.
\end{proof}

\section{Analysis of Intrinsic Dataset Property Characteristics (Intrinsic Dimension and Label Sharpness)}
\subsection{Invariance of Intrinsic Dataset Properties to Transformations}
\label{sec:app:resize}

In Fig. \ref{fig:res_depend}, left, we show that measured dataset intrinsic dimension $\dd$ estimates are barely affected by image resizing over a range of resolutions (square image sizes of $[32, 64, 128, 256, 512]$), with the specific example of $32\times32$ shown in the right of Fig. \ref{fig:dd_invar}. We show the similar result of measured dataset label sharpness $\KFh$ being invariant to image resizing in Fig. \ref{fig:res_depend}, right, and Fig. \ref{fig:kf_invar}, right, besides all datasets' $\KFh$ values being multiplied by the same positive constant (\ie, the relative placement of the $\KFh$ of each dataset stays the same with respect to such transformations). Because this constant is the same for all datasets for the given image resolution, it has no effect on the scaling law result of Eq. \eqref{eq:loss_scaling_dd}, as it can be folded into the constant $a$.

We show similar results for modifying the channel count of images (\ie, modifying all grayscale images to RGB) in the left of Figs. \ref{fig:dd_invar} and \ref{fig:kf_invar}.

\begin{figure}[!htbp]
\centering
\includegraphics[width=0.49\textwidth]{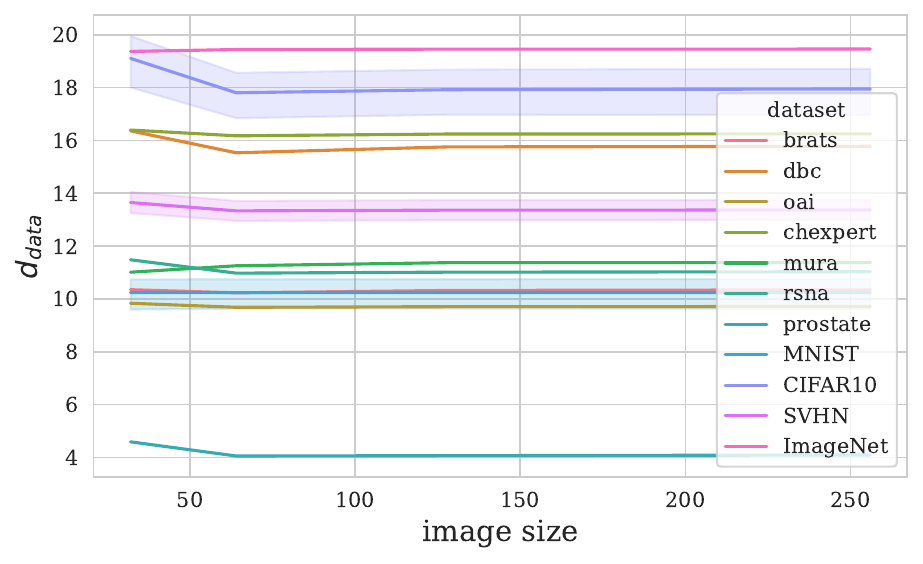}
\includegraphics[width=0.49\textwidth]{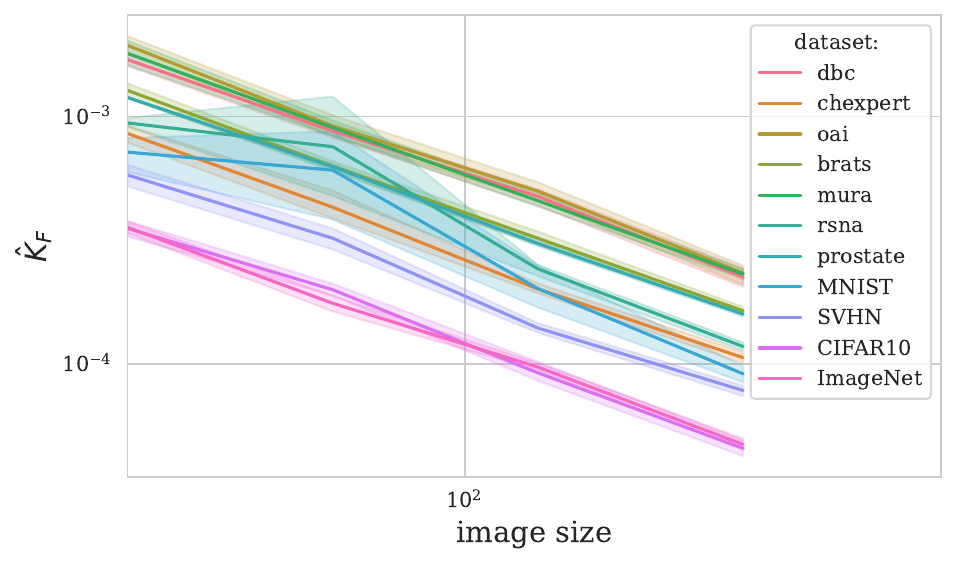}
    \caption{\textbf{Left:} Dependence of measured intrinsic dimension ($\dd$) of the image datasets which we analyze with respect to image size (height and width). $\dd$ values are averaged over all training set sizes; error bars indicate $95\%$ confidence intervals. \textbf{Right:} Same, but for measured dataset label sharpness $\KFh$. $\KFh$ values are averaged over all class pairings (Sec. \ref{sec:KFest}); error bars indicate $95\%$ confidence intervals.}
    \label{fig:res_depend}
\end{figure}

\begin{figure}[!htbp]
\centering
\includegraphics[width=0.45\textwidth]{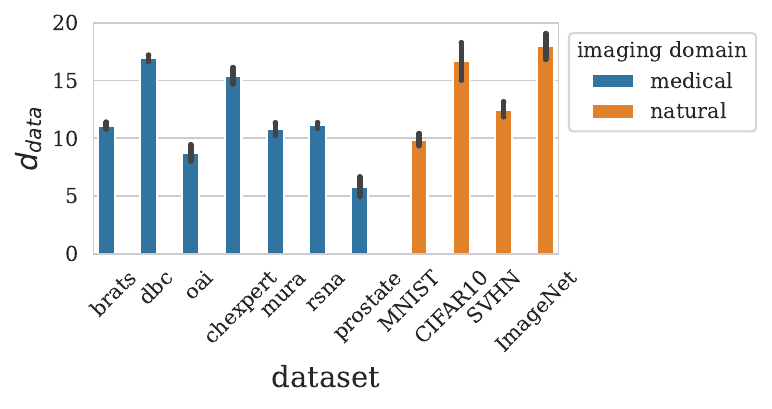}
\includegraphics[width=0.45\textwidth]{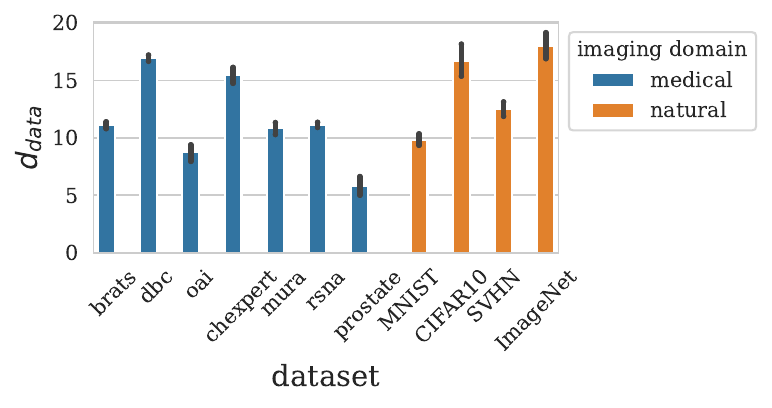}
    \caption{Measured intrinsic dimension ($\dd$) of the natural (orange) and medical (blue) image datasets which we analyze (Sec. \ref{sec:data}), for all images changed to RGB/$3$-channel (\textbf{left}), and all images resized to $32\times32$ (\textbf{right}). $\dd$ values are averaged over all training set sizes; error bars indicate $95\%$ confidence intervals. Compare to the default results in Fig. \ref{fig:kf_vs_dd}, left ($224\times224$, original image channel counts) for reference.}
    \label{fig:dd_invar}
\end{figure}

\begin{figure}[!htbp]
\centering
\includegraphics[width=0.45\textwidth]{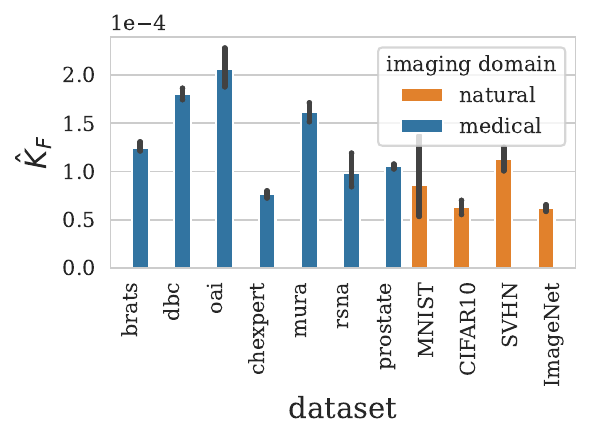}
\includegraphics[width=0.45\textwidth]{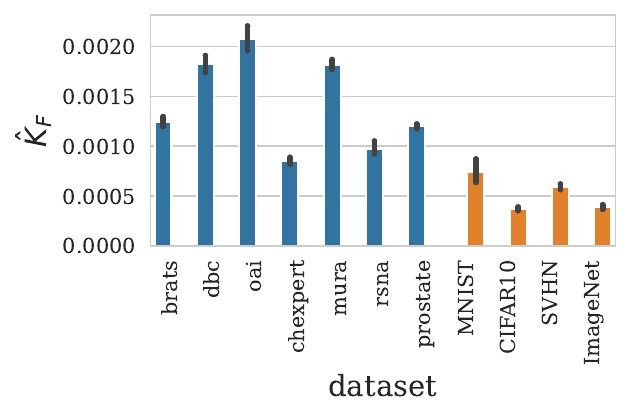}
    \caption{Measured label sharpnesses ($\KFh$) of the natural (orange) and medical (blue) image datasets which we analyze (Sec. \ref{sec:data}), for all images changed to RGB/$3$-channel (\textbf{left}), and all images resized to $32\times32$ (\textbf{right}). $\KFh$ values are averaged over all class pairings (Sec. \ref{sec:KFest}); error bars indicate $95\%$ confidence intervals. Compare to the default results in Fig. \ref{fig:kf_vs_dd}, right ($224\times224$, original image channel counts) for reference.}
    \label{fig:kf_invar}
\end{figure}

\section{Additional Results, Extensions, and Applications}

\subsection{Likelihood Analysis of Theoretical and Empirical Generalization Scaling Laws}
We hypothesized in the main text that the observed discrepancies in generalization scaling between natural and medical images with respect to intrinsic dataset dimension $\dd$ (Fig. \ref{fig:loss_vs_datadim}) were at least partially caused by the notable differences in dataset label sharpness $(\KF)$ between these two domains, indicated by our derived generalization scaling law of Equation \eqref{eq:loss_scaling_dd}. If we take Eq. \eqref{eq:loss_scaling_dd} as an equality (in other words, a model that can be regressed to the observed generalization data in Fig. \ref{fig:loss_vs_datadim}), we can analyze the likelihood that the observed shift between domains is caused by the scaling law's accounting for $\KF$ by seeing if the likelihood of our scaling law model (Model \textbf{A}) which accounts for $\KF$,
\begin{equation}
\label{eq:app:likemodeleqn}
y_{\mathbf{A}}(\dd, N, \KF ; a) := \log L \simeq -\frac{1}{\dd}\log N + \log \KF + a
\end{equation}
is higher than the likelihood of a model that does not account for $\KF$ (Model \textbf{B}),
\begin{equation}
y_{\mathbf{B}}(\dd, N; a) := \log L \simeq -\frac{1}{\dd}\log N + b.
\end{equation}
Here, recall that $L$ is the test loss of a trained network given the intrinsic dimension $\dd$ and label sharpness $\KF$ of the network's training dataset (Sec. \ref{sec:preliminaries}), and $N$ is the size of the training set. Each of the two scaling law models \textbf{A} and \textbf{B} will be fit to the observed generalization scaling data $D$: $D=\{(L; \dd, N, \KF)_i\}_{\forall i}$ for model \textbf{A} and $D=\{(L; \dd, N)_i\}_{\forall i}$ for model \textbf{B}, using all result data $i$ for a given network architecture (\ie, the datapoints in Fig. \ref{fig:loss_vs_datadim}); the fitted parameters are $a$ and $b$, for each respective model. We obtained these fitted models using SciPy's \texttt{curve\_fit} function \citep{2020SciPy-NMeth}, resulting in best-fit parameters of $\hat{a}$ and $\hat{b}$.

The likelihood ratio between two models is a well-known statistical test for determining the model that better explains the observed data \citep{vuong1989likelihood}, and is defined by $\mathcal{R}:={p(D | \textrm{model \textbf{A}})}/{p(D | \textrm{model \textbf{B}})}$. For such regression problems, the likelihood ratio is evaluated as 
\begin{equation}
\displaystyle
    \mathcal{R} = \frac{p(D | \textrm{model \textbf{A}})}{p(D | \textrm{model \textbf{B}})} = \frac{\exp\left[-\frac{1}{2}\sum_i(\log L_i - y_{\mathbf{A}}(d_{\mathrm{data}, i}, N_i, K_{\mathcal{F}, i} ; \hat{a}))^2\right]}{\exp\left[-\frac{1}{2}\sum_i(\log L_i - y_{\mathbf{B}}(d_{\mathrm{data}, i}, N_i; \hat{b}))^2\right]}.
\end{equation}

Here, $\log\mathcal{R} > 0 $ will indicate that model \textbf{A} explains the data better, $\log\mathcal{R} < 0$ will indicate that model \textbf{B} explains the data better, and  $\log\mathcal{R}\approx 0$ indicates that neither model is preferred.

As shown in Table \ref{tab:likelihoodcompare_dd}, we found that $\log\mathcal{R} > 0$ by a large margin for all network architectures, supporting the importance of accounting for $\KF$ in the scaling law, due to the variability of it across different domains.
These results seem reasonable because as shown in Fig \ref{fig:loss_vs_datadim}, there is a visible separation between the loss curves for the domains of natural and medical images. Allowing the scaling law to account for the label sharpness $\KF$ of the dataset will make it more accurate because different datasets possess different $\KF$ values (Fig. \ref{fig:kf_vs_dd}), and by Equation \eqref{eq:app:likemodeleqn}, different $\KF$ values will move the loss curve up and down.

\begin{table}[!htbp]
\centering
\begin{tabular}{ccc|ccc}
\hline
    ResNet-18 & ResNet-34 & ResNet-50 & VGG-13 & VGG-16 & VGG-19 \\
    \hline
    $13.5$ & $7.6$ & $11.7$ & $8.1$ & $10.5$ & $12.3$  \\
    \hline
    \end{tabular}
\caption{Log-ratio $\log\mathcal{R}$ between (\textbf{A}) the likelihood of the network generalization $\dd$ scaling law model that accounts for label sharpness, and (\textbf{B}) the likelihood of the scaling law model that does not, given generalization data observed in our experiments (Fig. \ref{fig:loss_vs_datadim}), for each network architecture.}
\label{tab:likelihoodcompare_dd}
\end{table}

% We can perform the same analysis for the generalization scaling law with respect to learned representation intrinsic dimension ($\dr$) (Eq. \eqref{eq:dr_logscaling} and Fig. \ref{fig:loss_vs_reprdim}); the results are this are shown in Table \ref{tab:likelihoodcompare_dd}.

% \begin{table}[!htbp]
% \centering
% \begin{tabular}{ccc|ccc}
% \hline
%     ResNet-18 & ResNet-34 & ResNet-50 & VGG-13 & VGG-16 & VGG-19 \\
%     \hline
%     $19.1$ & $13.6$ & $17.2$ & $13.0$ & $14.2$ & $14.6$  \\
%     \hline
%     \end{tabular}
% \caption{Log-ratio $\log\mathcal{R}$ between (\textbf{A}) the likelihood of the network generalization $\dr$ scaling law model that accounts for label sharpness, and (\textbf{B}) the likelihood of the scaling law model that does not, given generalization data observed in our experiments (Fig. \ref{fig:loss_vs_reprdim}), for each network architecture.}
% \label{tab:likelihoodcompare_dr}
% \end{table}

\subsection{Evaluating a Dataset from an Additional Domain}
\label{app:isic}

In this section, we extend our analysis to a new dataset from a third domain beyond natural images and radiology images, in order to determine whether our hypotheses extend to other domains (\eg, that dataset label sharpness is related to which domain the dataset is within). We use the ISIC skin lesion image dataset of \cite{codella2018skin}, which interestingly, has certain characteristics that both natural and radiological images share, such as being RGB photographs (like natural images), and having standardized acquisition procedure and object framing for the purpose of clinical tasks (like radiological images). For all experiments we use the task/labeling for melanocytic nevus detection.

First, we find that ISIC has an intrinsic dimension $\dd\simeq 12$ that is in between typical natural image dataset $\dd$ values and typical radiology dataset $\dd$ values (Fig. \ref{fig:kf_vs_dd_isic}, left). We similarly see that its label sharpness $\KFh\simeq 10^{-4}$ is in the upper end of typical natural image dataset $\KFh$ values, and below all radiology dataset $\KFh$ (Fig. \ref{fig:kf_vs_dd_isic}, right). It makes intuitive sense that these intrinsic properties of the ISIC dataset are in between the two domains of natural and radiological images, given the aforementioned characteristics of images from both domains that it possesses. 

\begin{figure}[!htbp]
\centering
\includegraphics[width=0.54\textwidth]{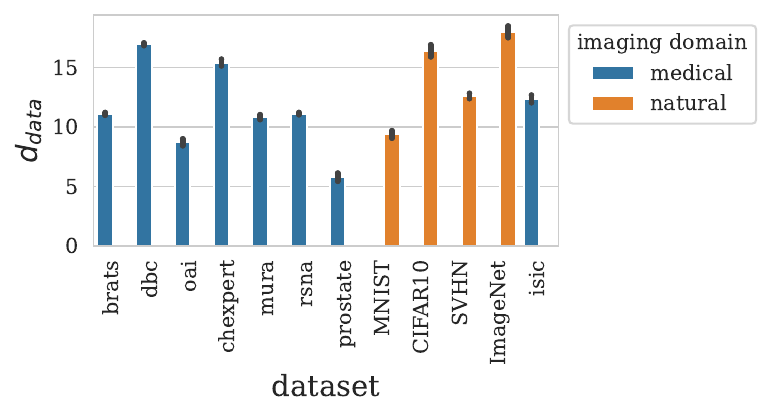}
\includegraphics[width=0.41\textwidth]{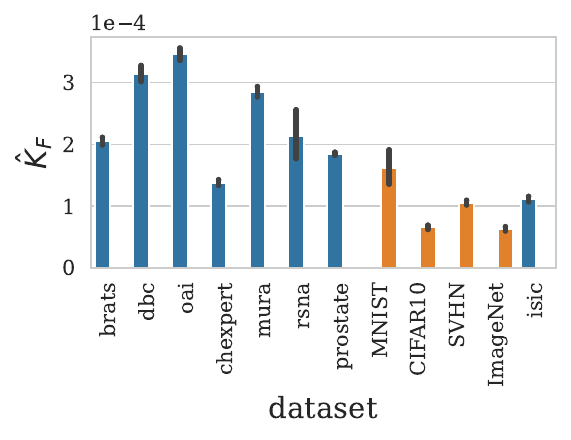}
    \caption{Measured intrinsic dimension ($\dd$, \textbf{left}) and label sharpnesses ($\KFh$, \textbf{right}) of the natural (orange) and medical (blue) image datasets which we analyze (Sec. \ref{sec:data}), \textbf{with the ISIC dataset included on the right of both figures.}
    % $\KFh$ and $\dd$ characterize the two imaging domains, which we will show dictate different facets of learning behavior. 
    $\dd$ values are averaged over all training set sizes, and $\KFh$ over all class pairings (Sec. \ref{sec:KFest}); error bars indicate $95\%$ confidence intervals.}
    \label{fig:kf_vs_dd_isic}
\end{figure}

We next performed the same generalization experiments as in the main text for ISIC, training each network model for the assigned task with $N=1750$. Given our generalization scaling law of Eq. \eqref{eq:loss_scaling_dd}, ISIC having a $\KF$ value between the typical respective values of natural and radiological domains would imply that models trained on the dataset would have test loss values between the models trained on these two domains, given ISIC's $\dd$. We see in Fig. \ref{fig:loss_vs_datadim_isic} that this was indeed the case for all network architectures; the generalization ability of the ISIC models (indicated by purple circles) are between the typical generalization curves of natural image models and radiological image models.

\begin{figure}[!htbp]
\centering
\includegraphics[width=0.98\textwidth]{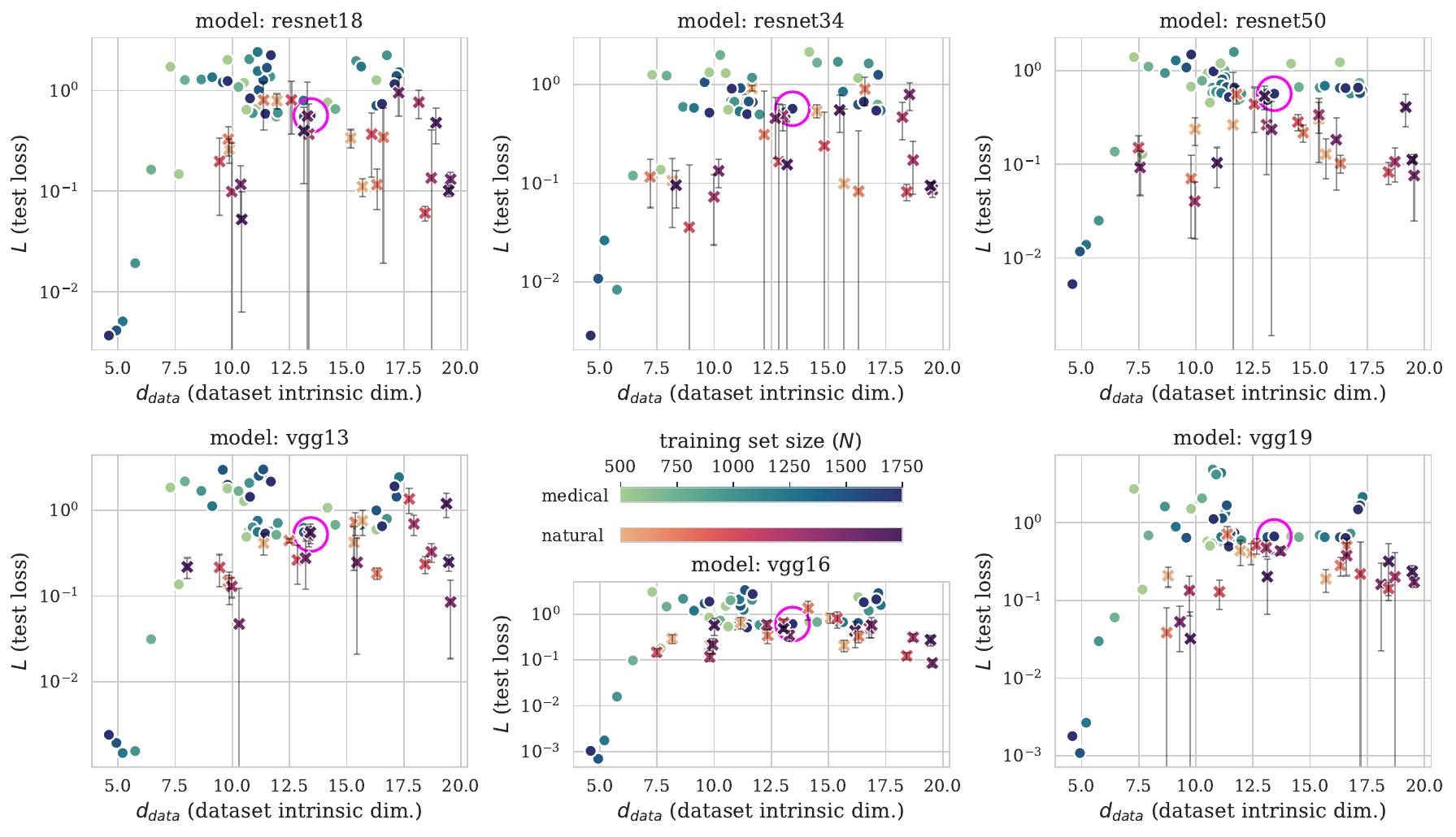}
\caption{Same as Fig. \ref{fig:loss_vs_datadim}, but with ISIC dataset results added with purple circles.
% \textbf{Bottom:} same but for test accuracy.
}
\label{fig:loss_vs_datadim_isic}
\end{figure}

Moreover, the ``in-between'' $\KF$ of ISIC also implies that models trained on this dataset would be more adversarially robust than the radiological image models (with their high dataset $\KF$ values), yet less robust than the natural image models (with their low dataset $\KF$) (Theorem \ref{thm:advscaling_KF}). In Fig. \ref{fig:atk_vs_kf_isic} we see that this is the case for some network architectures, while for others, ISIC models (purple circles) end up close to the natural image models.

\begin{figure}[!htbp]
\centering
\includegraphics[width=0.7\textwidth]{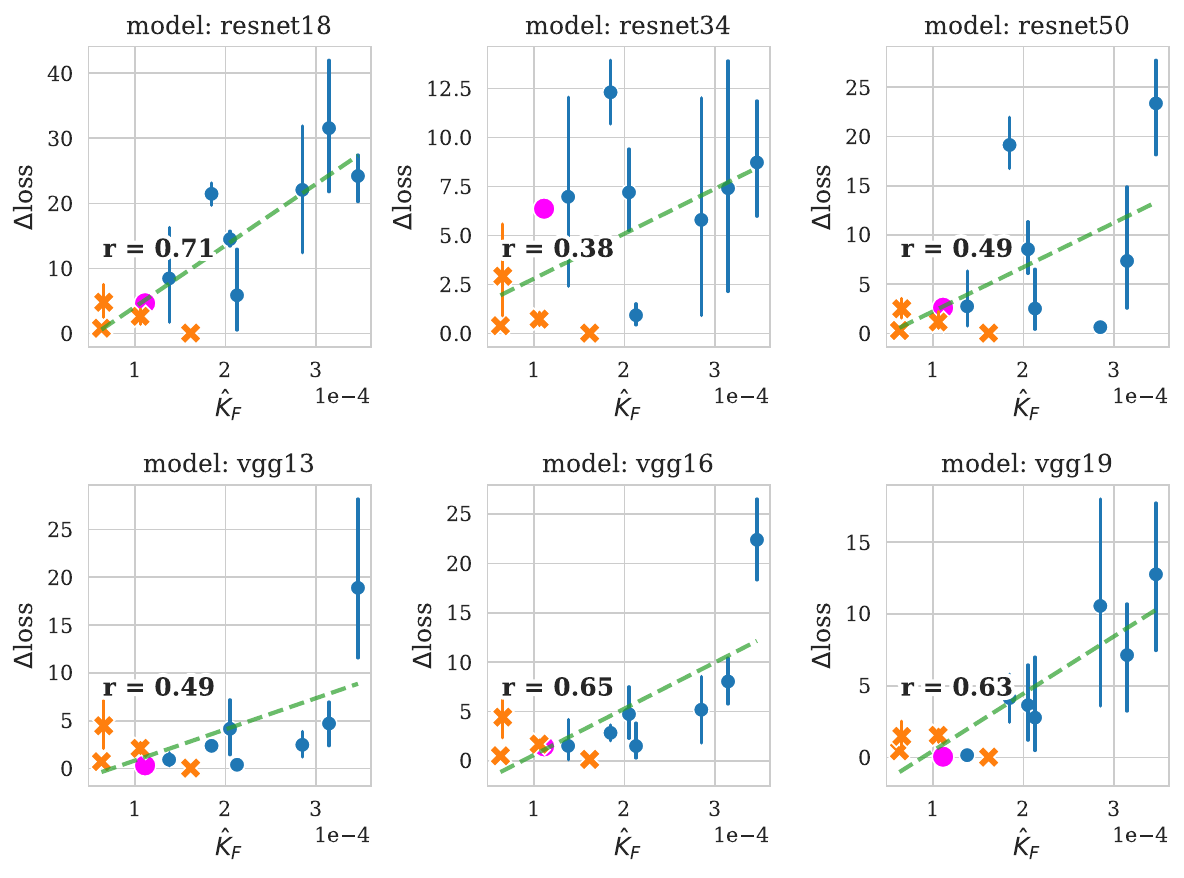}
    \caption{Same as Fig. \ref{fig:atk_vs_kf}, but with ISIC dataset results added with purple circles.}
    \label{fig:atk_vs_kf_isic}
\end{figure}

\subsection{Practical Application: Task Selection for Medical Images}
\label{sec:app:task_select}

In this section we will demonstrate a practical usage of our formalism. It is common for new medical image datasets to come equipped with many different labels provided by clinical annotators, prior to any attempt to train a model to learn to make such predictions from the data. The question we examine in this section is: \textit{given a new dataset with a variety of image labels, which tasks will be easier for a model learn, and which will be harder?} This is an important question to guide the model development process of practitioners who wish to take the first steps of training models for automated diagnosis of a new dataset and/or modality, the answer of which may not be clear solely from the visible image characteristics.

For example, the RSNA-IH-CT dataset (Sec. \ref{sec:data}) was annotated with labels for different types of hemorrhages, but some could be easier to detect than others. Consider that we wish to decide whether to train a binary classification model to (1) detect \textit{any} type of hemorrhage out of 5 sub-types or (2) detect a specific type, such as epidural hemorrhage. Na\"ively, it may seem that the second task is more specific and therefore may be more challenging, yet if some visual characteristic makes epidural hemorrhages easily noticeable, the first task could be more challenging, as it requires learning to differentiate between (a) healthy cases and (b) each type of hemorrhage. We can get a general idea for the relative difficulty of these two tasks using our derived scaling law, as follows.

Let's say that we wish to estimate which task is likely to be more challenging for a given model to learn by determining which has the higher expected test loss $L$. Our scaling law (Eq. \eqref{eq:loss_scaling_dd}) estimates that $L\simeq \mathcal{O}(\KF N^{-1/\dd})$, but because the equation is a bound (not an equality), estimating absolute test loss values is not feasible. However, if we instead consider the \textit{ratio} of test losses for two different possible tasks on the same dataset, a prediction is more tractable. While $N$ and $\dd$ are both independent of task choice, the label sharpness $\KF$ (Sec. \ref{sec:KFest}) will change depending on the labels assigned to the data for the given task, which can be quickly measured from the dataset without any model training. If we take $\KF^{(1)}$ and $L^{(1)}$ to be the measured label sharpness and expected test loss for the first task (detection of any hemorrhage), respectively, and likewise for $\KF^{(2)}$ and $L^{(2)}$ for the second task (epidural hemorrhage detection), we get that approximately,
\begin{equation}
\label{eq:taskpred}
    \frac{L^{(1)}}{L^{(2)}} \underset{\sim}{\propto} \frac{\KF^{(1)} N^{-1/\dd}}{\KF^{(2)} N^{-1/\dd}} = \frac{\KF^{(1)}}{\KF^{(2)}},
\end{equation}
implying that the task with the higher $\KF$ will likely be more challenging for the model (higher test loss $L$).

To test this, we measured $\KFh^{(1)} = 2.1 \pm 0.4 \times 10^{-4}$ and $\KFh^{(2)} = 1.45 \pm 0.06 \times 10^{-4}$ for the two respective tasks (95\% CI over 25 evaluations of $M^2$ pairings $M=1000$, as in Sec. \ref{sec:KFest} and Fig. \ref{fig:kf_vs_dd}). Although approximate, Eq. \eqref{eq:taskpred} indicates that task 2 will be easier. We then trained each of our evaluated models for each of the two tasks, with results shown in Table \ref{tab:taskcompare} ($N=1750$ and all other training details are the same as for the main paper experiments). We see that all models obtained lower test loss on task 2 than on task 1, and similarly obtained higher test accuracy, indicating that task 2 was indeed easier.

\begin{table}[!htbp]
\fontsize{9pt}{9pt}\selectfont
\centering
\begin{tabular}{l||ccc|ccc||c}
\hline
     & ResNet-18 & ResNet-34 & ResNet-50 & VGG-13 & VGG-16 & VGG-19  & $\KFh$ \\
    \hline
    \textbf{Task 1} & $1.29$ & $1.23$ & $1.03$ & $0.69$ & $0.50$ & $0.51$ & $2.1 \pm 0.4$ \\
    \textbf{Task 2} & $0.64$ & $0.66$ & $0.66$ & $0.63$ & $0.62$ & $0.90$ & $1.45 \pm 0.06$ \\
    \hline
    \textbf{Task 1} & $76\%$ & $74\%$  & $74\%$ & $73\%$ & $75\%$ & $76\%$ & $2.1 \pm 0.4$ \\
    \textbf{Task 2} & $80\%$ & $83\%$ & $83\%$  & $85\%$ & $82\%$ & $81\%$ & $1.45 \pm 0.06$ \\
    \hline
    \end{tabular}
\caption{\textbf{Top section:} Test set \textbf{loss} for each model trained on each of the two hemorrhage detection tasks, alongside the measured label sharpness $\KFh$ for each task (Task 1 is detecting any hemorrhage, Task 2 is detecting epidural hemorrhage). \textbf{Bottom section:} Same, but for test set \textbf{accuracy}.}
\label{tab:taskcompare}
\end{table}

Note that Equation \eqref{eq:taskpred} is just an approximation, and that tasks with more similar measured $\KF$ values for the same dataset could be harder to distinguish. Of course, this experiment is just an example, and future study with other datasets is warranted.

\subsection{Evaluation at Much Higher Training Set Sizes}

While many of our datasets do not support going to substantially higher training set sizes than our main experiments' maximum of $N=1750$ (see Sec. \ref{sec:data}), we can still evaluate the generalization scaling of models training on two datasets that do allow for significantly higher $N$. To this end, we trained each of our six models on the CheXpert medical image dataset and on the CIFAR-10 natural image dataset (for classes 1 and 2) at the highest training set size possible for binary classification on these datasets, $N=9250$. We would expect from our generalization scaling law (Eq. \eqref{eq:loss_scaling_dd}), that for a fixed dataset (and therefore $\dd$ and $\KF$) and architecture, the loss would decrease with higher $N$. The results of this are shown in Tables \ref{tab:highN_cifar} and \ref{tab:highN_chexpert} below; we see that this is indeed the case for all models (lower loss for higher training set size). We also see that the general trend of the natural image models having much lower loss than the medical image models is maintained, even though these two datasets have similar intrinsic dimensions ($d_{data}\simeq 15-17$).

\begin{table}[!htbp]
\centering
\begin{tabular}{@{}l|llllll@{}}
$N$ & ResNet-18 & ResNet-34 & ResNet-50 & VGG-13 & VGG-16 & VGG-19 \\ \hline
9250 & 0.1660 & 0.1821 & 0.1179 & 0.1086 & 0.1045 & 0.0828 \\
1000 & 0.5312 & 0.7402 & 0.5128 & 0.9764 & 0.6001 & 0.3974
\end{tabular}
\caption{Test losses for models trained on CIFAR-10 binary classification for high training set size $N=9250$ compared to those trained on $N=1000$.}
\label{tab:highN_cifar}
\end{table}

\begin{table}[!htbp]
\centering
\begin{tabular}{@{}l|llllll@{}}
$N$ & ResNet-18 & ResNet-34 & ResNet-50 & VGG-13 & VGG-16 & VGG-19 \\ \hline
9250 & 0.7712 & 0.6370 & 0.6789 & 0.6014 & 0.6014 & 0.6016 \\
1000 & 1.3479 & 0.7894 & 0.9793 & 0.6700 & 0.7409 & 0.6806
\end{tabular}
\caption{Test losses for models trained on CheXpert binary classification for high training set size $N=9250$ compared to those trained on $N=1000$.}
\label{tab:highN_chexpert}
\end{table}

\subsection{Dependence of Network Performance on Image Resolution}
It seems plausible that training a network to perform certain medical image binary classification tasks would be difficult at low image resolutions, due to the visual similarity of positive and negative images for some tasks (as any opposed to the typically low visual similarity of images from different classes in natural image datasets). To test this, we trained a ResNet-18 on each medical image dataset (with $N=1750$ and all other training settings at their defaults) over a wide range of image resolutions (square image sizes of $[32, 64, 128, 256, 512]$), to see if the test accuracy was smaller for low resolutions. The results are shown in Fig. \ref{fig:acc_vs_res}, and surprisingly, there is little performance drop for small resolutions. This may actually make sense, considering datasets like MedMNIST \citep{medmnistv2}, where training for a wide variety of medical image classification tasks is possible even at $28\times28$ resolution. Of course, this would probably not be the case for more fine-grained tasks such as semantic segmentation.

\begin{figure}[!htbp]
\centering
\includegraphics[width=0.7\textwidth]{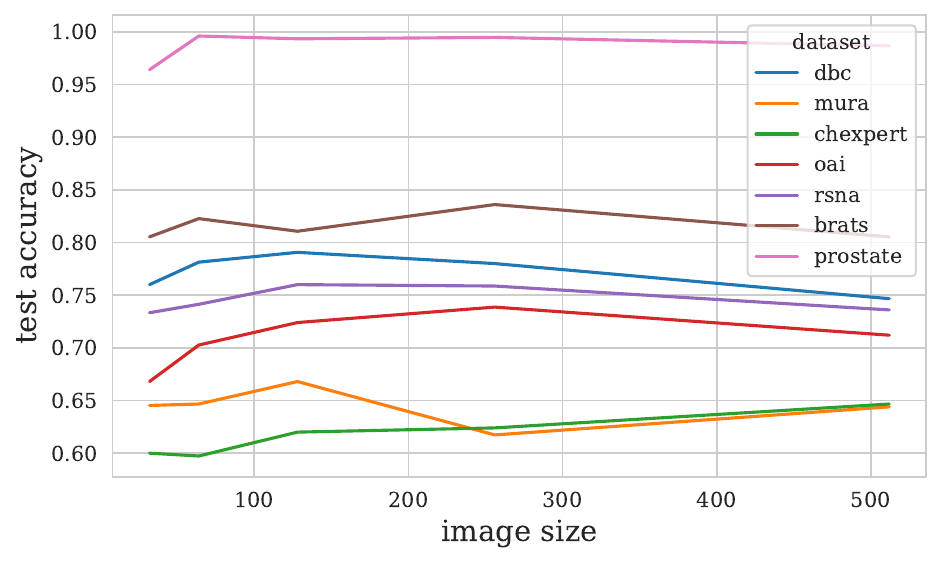}
    \caption{Dependence of network performance on image size for different medical image classification datasets (ResNet-18, training set size of $1750$).}
    \label{fig:acc_vs_res}
\end{figure}

\section{Additional Visualizations}

\subsection{Example Adversarial Attacks on Medical Images}
\label{sec:app:eg_atk_med}

We show example attacked medical images for each dataset in Fig. \ref{fig:eg_atks}.

\begin{figure}[!htbp]
\centering
\includegraphics[width=0.98\textwidth]{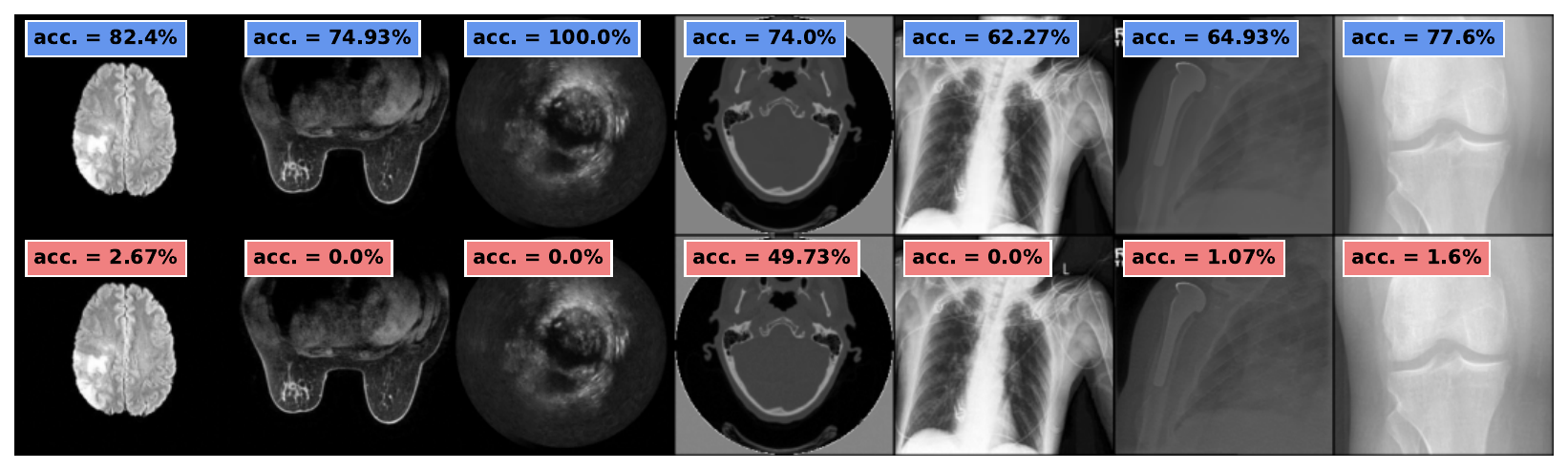}
    \caption{\textbf{Susceptibility of medical images to adversarial attack.} \textbf{Top row:} test set prediction accuracy of models trained on each medical image dataset for its corresponding diagnostic task (Sec. \ref{sec:data}), with example test images shown. \textbf{Bottom Row:} accuracies after each test set was attacked by FGSM ($\epsilon=2/255$), with example attacked images shown. The models are ResNet-18s with training set sizes of $N=1750$.}
    \label{fig:eg_atks}
\end{figure}

\section{Main Results with Other Metrics}

In this section we will show our main results but with other metrics for generalization, adversarial robustness, and/or intrinsic dimensionality.

\subsection{Generalization Scaling with $\dd$ and $\dr$}
\label{app:moreresults:scaling}

\paragraph{Continuation of Sec. \ref{sec:exp:ddata_scaling}.} In Fig. \ref{fig:acc_vs_datadim} we show the scaling of test \textit{accuracy} with intrinsic dataset dimension $\dd$, using the default MLE estimator (Sec. \ref{sec:dest}). In Figs. \ref{fig:loss_vs_datadim_twonn} and \ref{fig:acc_vs_datadim_twonn} we show the scaling of test loss and accuracy, respectively, but instead using TwoNN (Sec. \ref{sec:dest}) to estimate $\dd$.

\begin{figure}[!htbp]
\centering
\includegraphics[width=0.98\textwidth]{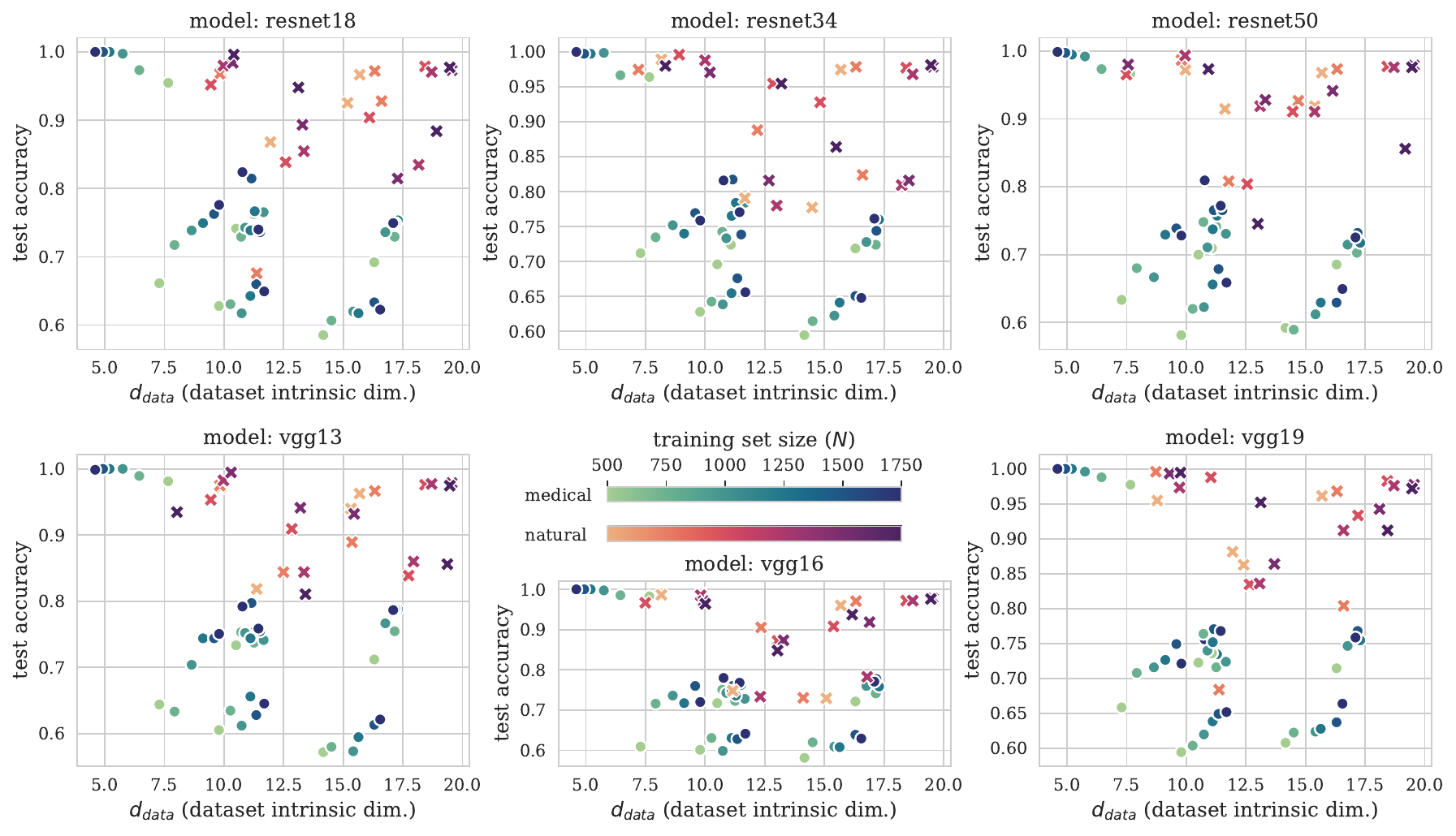}
    \caption{Scaling of test accuracy/generalization ability with training set intrinsic dimension ($\dd$) for natural and medical datasets.}
    \label{fig:acc_vs_datadim}
\end{figure}

\begin{figure}[!htbp]
\centering
\includegraphics[width=0.98\textwidth]{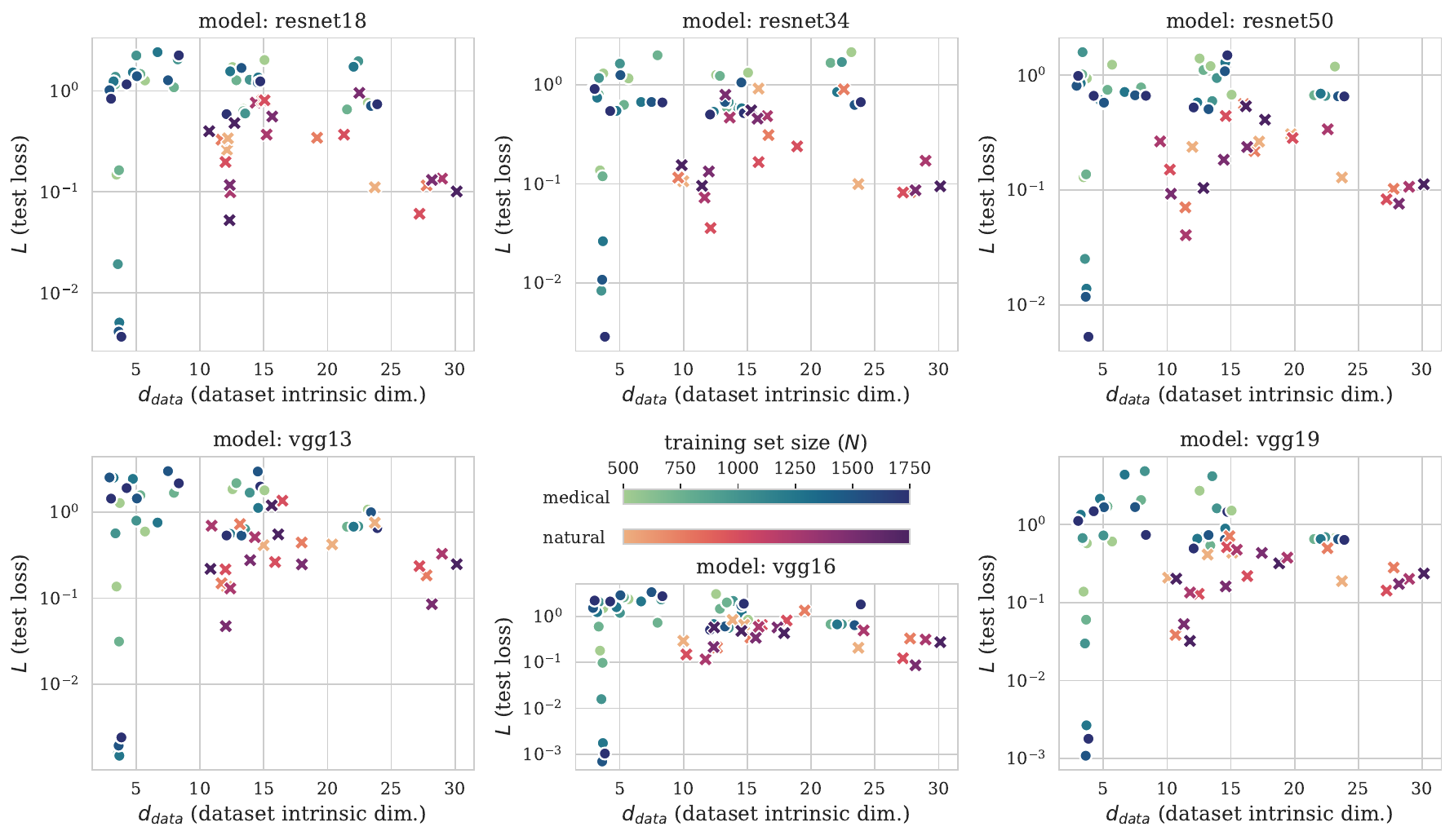}
    \caption{Scaling of log test loss/generalization ability with training set intrinsic dimension ($\dd$) for natural and medical datasets, with $\dd$ computed via TwoNN \citep{facco2017estimating}.}
    \label{fig:loss_vs_datadim_twonn}
\end{figure}

\begin{figure}[!htbp]
\centering
\includegraphics[width=0.98\textwidth]{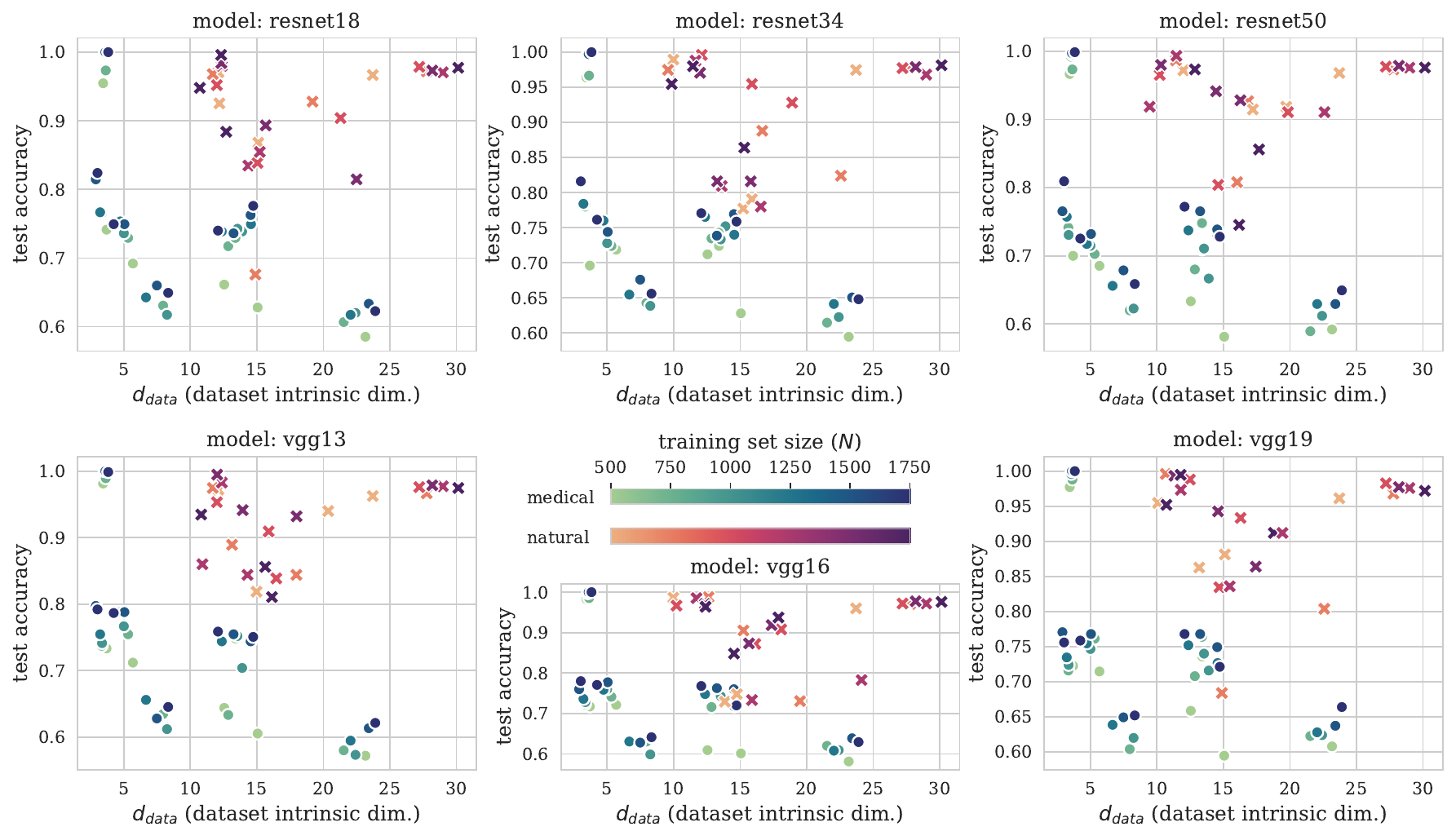}
    \caption{Scaling of test accuracy/generalization ability with training set intrinsic dimension ($\dd$) for natural and medical datasets, with $\dd$ computed via TwoNN \citep{facco2017estimating}.}
    \label{fig:acc_vs_datadim_twonn}
\end{figure}

\paragraph{Continuation of Sec. \ref{sec:exp:drepr_scaling}.} Next, in Fig. \ref{fig:acc_vs_reprdim} we show the scaling of test \textit{accuracy} with learned representation intrinsic dimension $\dr$, using the default TwoNN estimator (Sec. \ref{sec:dest}). In Figs. \ref{fig:loss_vs_reprdim_twonn} and \ref{fig:acc_vs_reprdim_twonn} we show the scaling of test loss and accuracy, respectively, but instead using MLE (Sec. \ref{sec:dest}) to estimate $\dr$.

\begin{figure}[!htbp]
\centering
\includegraphics[width=0.98\textwidth]{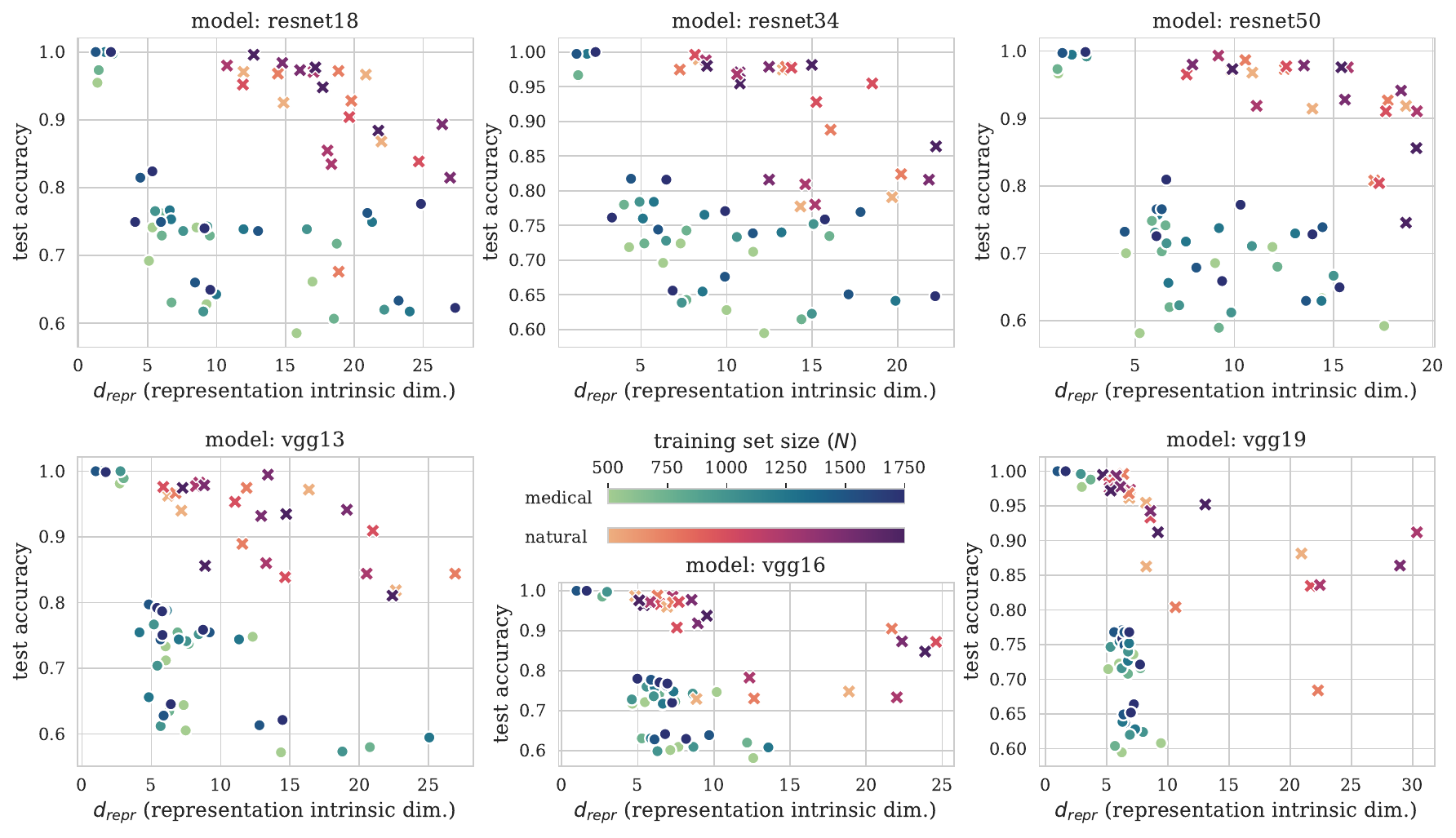}
    \caption{Scaling of test accuracy/generalization ability with the intrinsic dimension of final hidden layer learned representations of the training set ($\dr$) for natural and medical datasets.}
    \label{fig:acc_vs_reprdim}
\end{figure}

\begin{figure}[!htbp]
\centering
\includegraphics[width=0.98\textwidth]{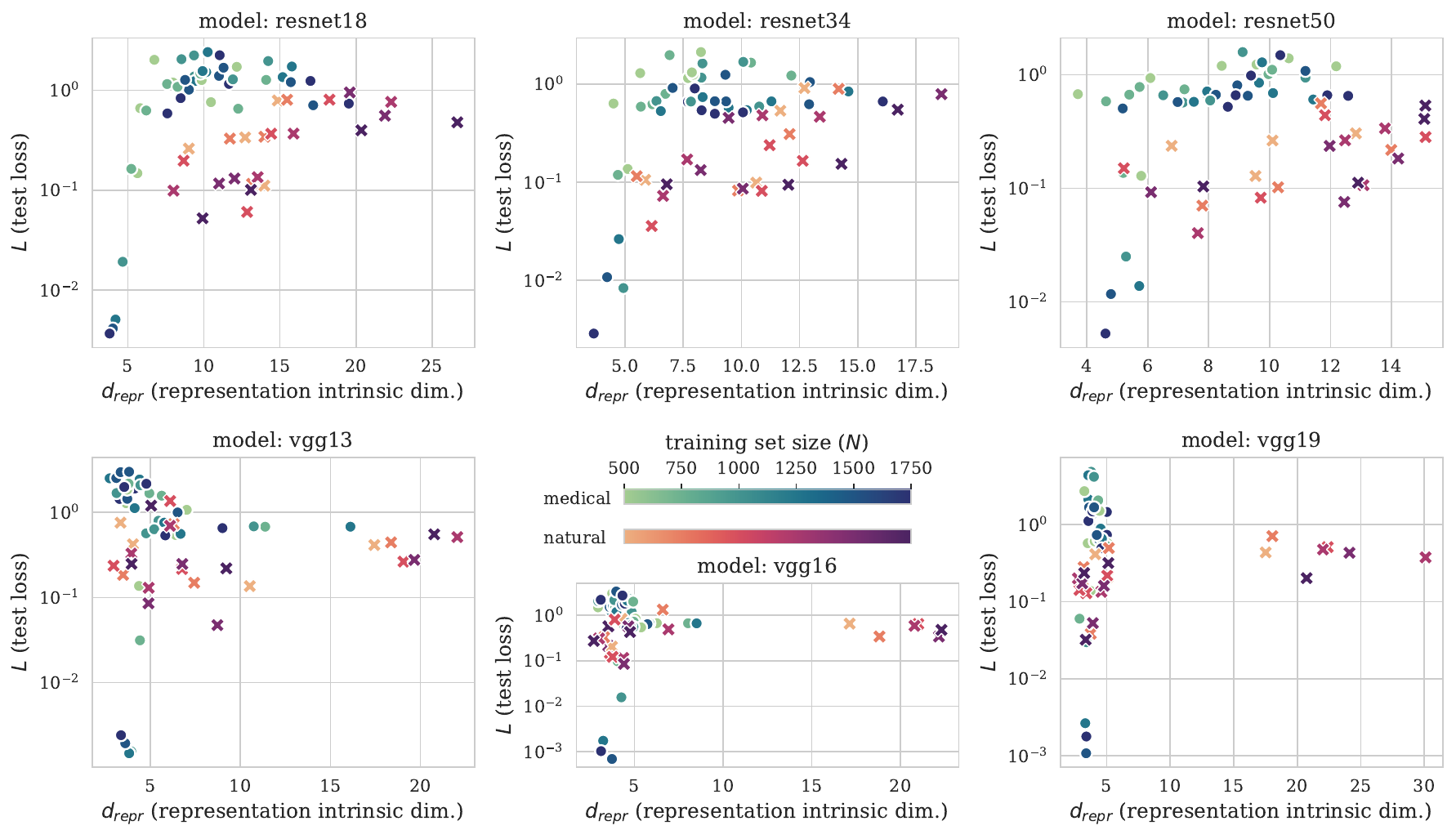}
    \caption{Scaling of log test loss/generalization ability with the intrinsic dimension of final hidden layer learned representations of the training set ($\dr$) for natural and medical datasets, with $\dd$ computed via MLE (Sec. \ref{sec:dest}).}
    \label{fig:loss_vs_reprdim_twonn}
\end{figure}

\begin{figure}[!htbp]
\centering
\includegraphics[width=0.98\textwidth]{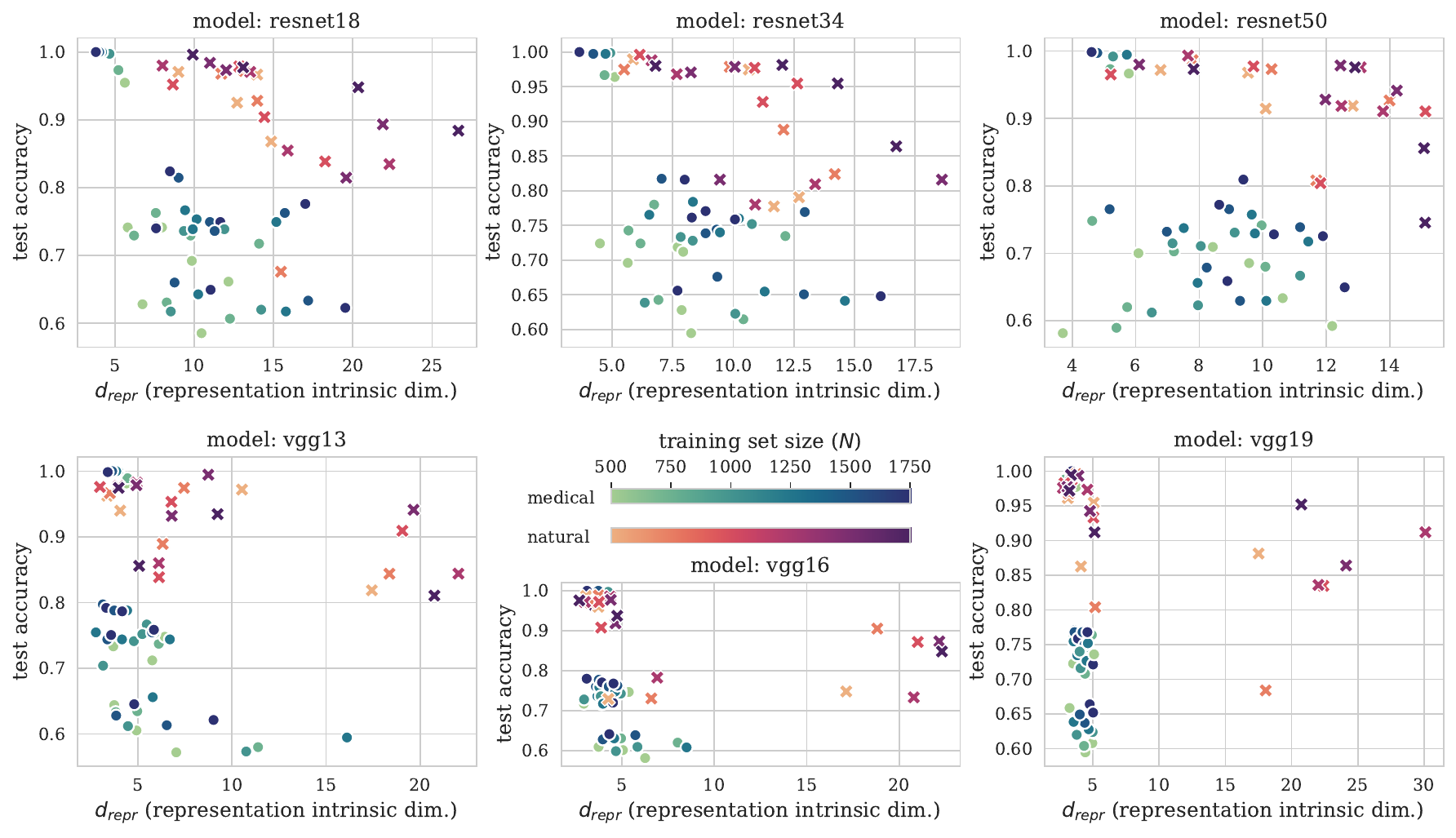}
    \caption{Scaling of test accuracy/generalization ability with the intrinsic dimension of final hidden layer learned representations of the training set ($\dr$) for natural and medical datasets, with $\dd$ computed via MLE (Sec. \ref{sec:dest}).}
    \label{fig:acc_vs_reprdim_twonn}
\end{figure}
% \begin{figure}[!htbp]
% \centering
% \includegraphics[width=0.98\textwidth]{figs/repr_dim_acc_scaling.pdf}
%     \caption{Scaling of test accuracy/generalization ability with the intrinsic dimension of final hidden layer learned representations of the training set ($\dr$) for natural and medical datasets.}
%     \label{fig:acc_vs_datadim}
% \end{figure}
\subsection{Bounding Hidden Representation Intrinsic Dimension with Dataset Intrinsic Dimension}
\label{app:moreresults:ddvsdr}

In Fig. \ref{fig:dr_vs_ddtwonn} we show the $\dd$ vs. $\dr$ results as in Fig. \ref{fig:dr_vs_dd}, but with dimensionality estimates computed with TwoNN instead of MLE (Sec. \ref{sec:dest}).

\begin{figure}[!htbp]
\centering
    \includegraphics[width=0.48\textwidth]{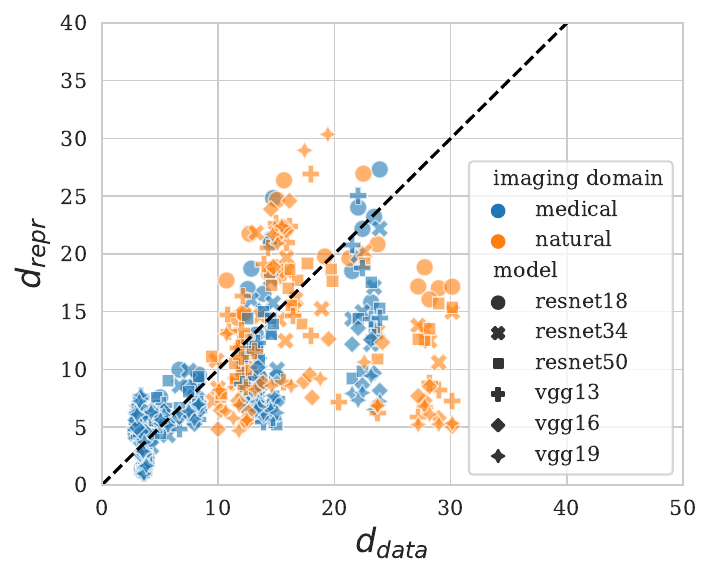}
    \caption{Training dataset intrinsic dimension $\dd$ vs. learned representation intrinsic dimension $\dr$, both computed using TwoNN instead of MLE (Sec. \ref{sec:dest}). Each point corresponds to a (model, dataset, training set size) combination.}
    \label{fig:dr_vs_ddtwonn}
\end{figure}

\subsection{Adversarial Robustness scaling with $\KFh$}
\label{app:moreresults:robust}
\paragraph{Continuation of Sec. \ref{sec:exp:robust}.} In Figs. \ref{fig:atk_vs_kfeps1}, \ref{fig:atk_vs_kfeps4} and \ref{fig:atk_vs_kfeps8}, we show the scaling of test loss penalty due to FGSM adversarial attack with respect to measured dataset label sharpness $\KFh$, for attack $\epsilon$ of $1/255$, $4/255$, and $8/255$, respectively. In Figs \ref{fig:atk_vs_kf_acc_eps1}, \ref{fig:atk_vs_kf_acc_eps2}, \ref{fig:atk_vs_kf_acc_eps4} and \ref{fig:atk_vs_kf_acc_eps8} we instead show the scaling of test \textit{accuracy} penalty, for each FGSM attack $\epsilon$ of $1/255$, $2/255$, $4/255$, and $8/255$, respectively. Finally, in Tables \ref{tab:adv_medonly} and \ref{tab:adv_natonly} we report per-domain correlations of loss penalty and dataset $\KF$, for medical images and natural images respectively.

\begin{figure}[!htbp]
\centering
\includegraphics[width=0.7\textwidth]{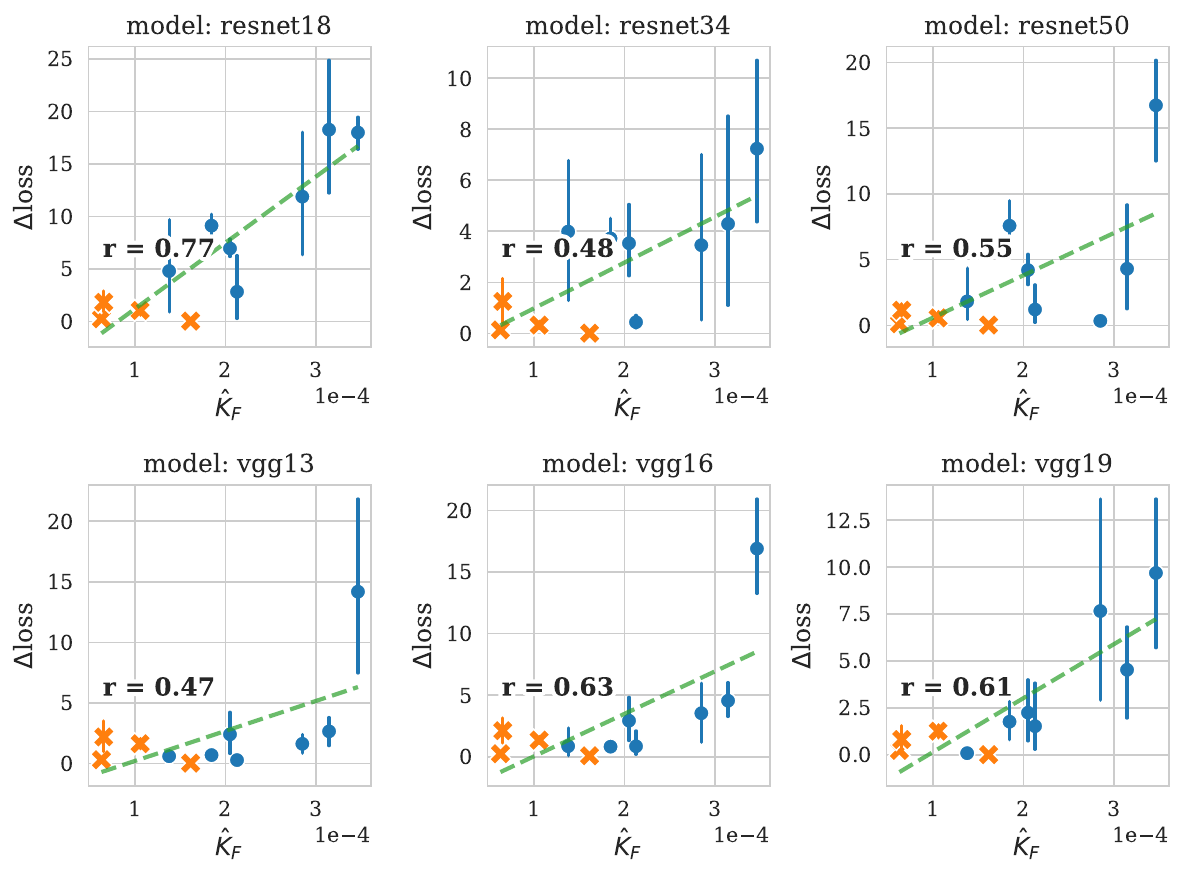}
    \caption{Scaling of test set loss penalty due to $\epsilon=1/255$ FGSM adversarial attack with dataset label sharpness $\KF$ for natural (orange) and medical (blue) datasets.}
    \label{fig:atk_vs_kfeps1}
\end{figure}

\begin{figure}[!htbp]
\centering
\includegraphics[width=0.7\textwidth]{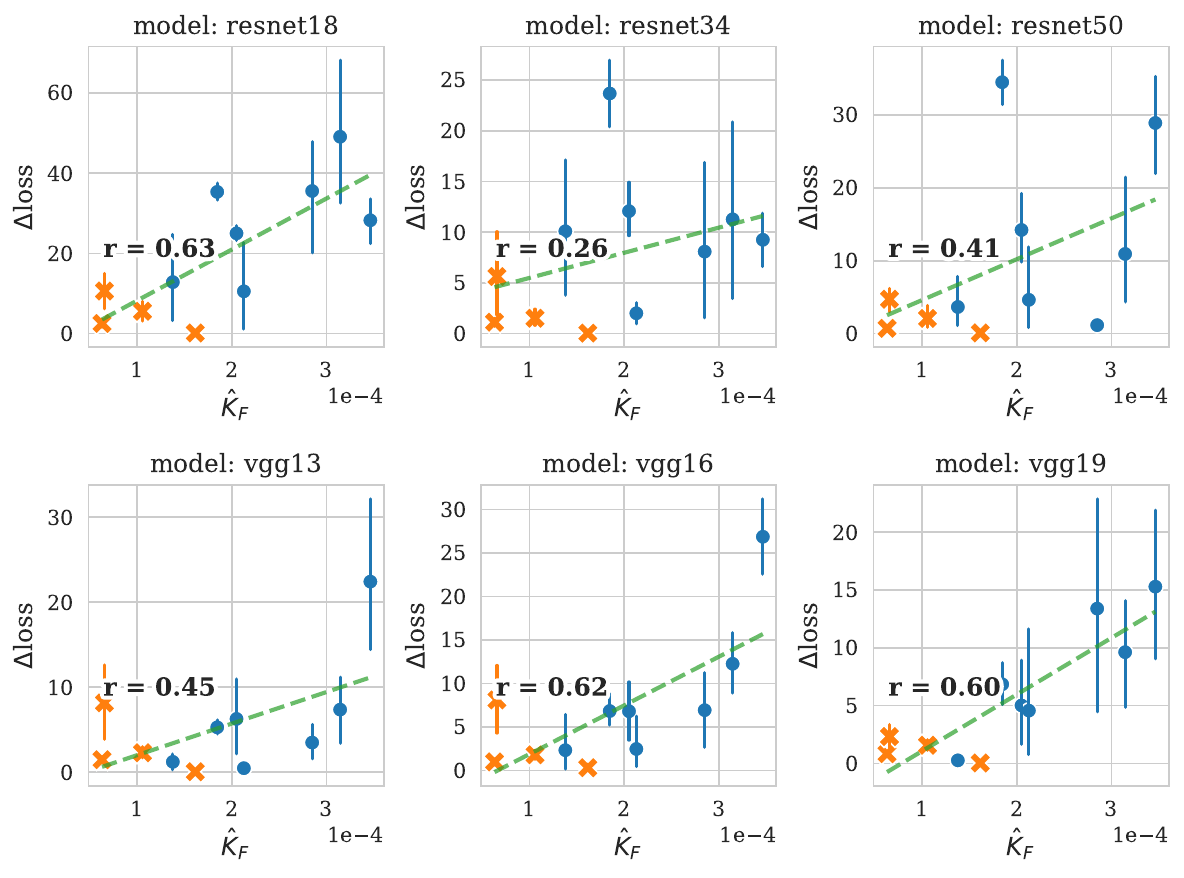}
    \caption{Scaling of test set loss penalty due to $\epsilon=4/255$ FGSM adversarial attack with dataset label sharpness $\KF$ for natural (orange) and medical (blue) datasets.}
    \label{fig:atk_vs_kfeps4}
\end{figure}

\begin{figure}[!htbp]
\centering
\includegraphics[width=0.7\textwidth]{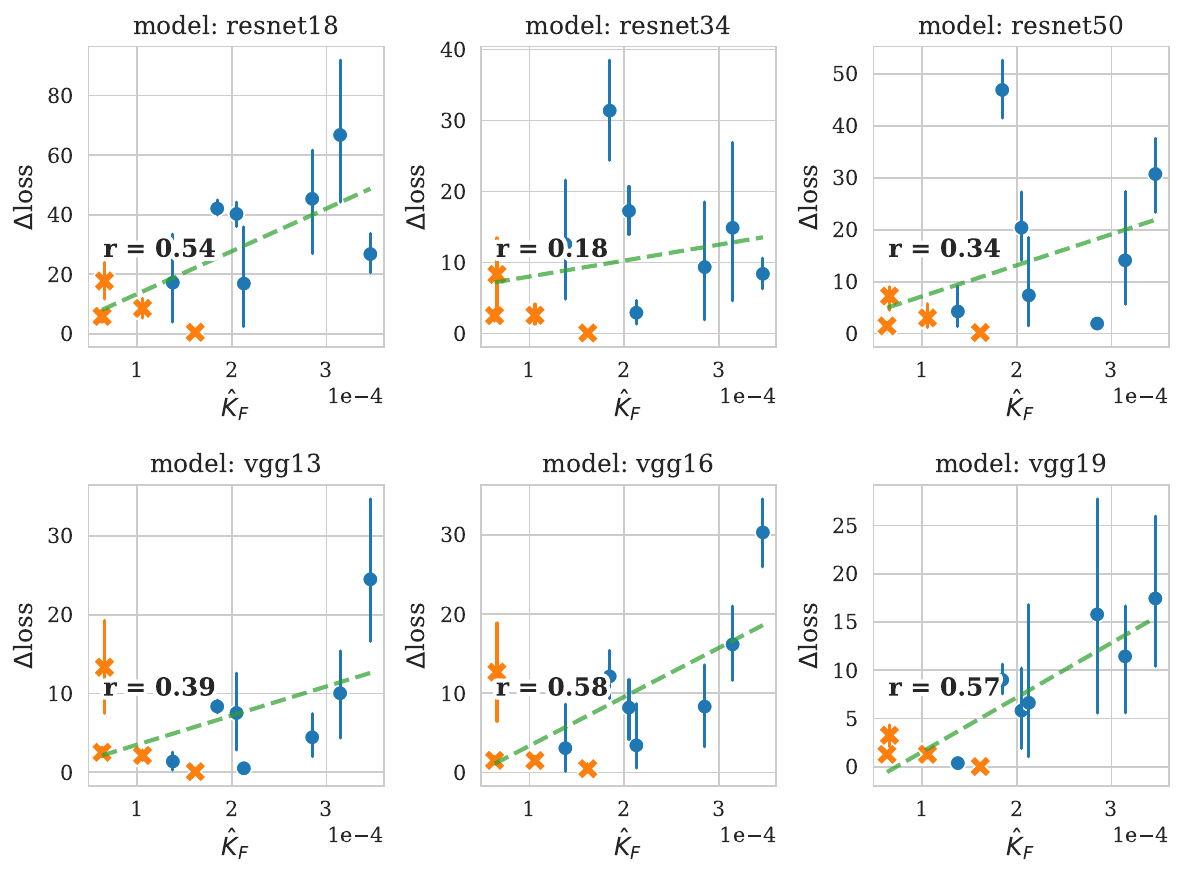}
    \caption{Scaling of test set loss penalty due to $\epsilon=8/255$ FGSM adversarial attack with dataset label sharpness $\KF$ for natural (orange) and medical (blue) datasets.}
    \label{fig:atk_vs_kfeps8}
\end{figure}

\begin{figure}[!htbp]
\centering
\includegraphics[width=0.7\textwidth]{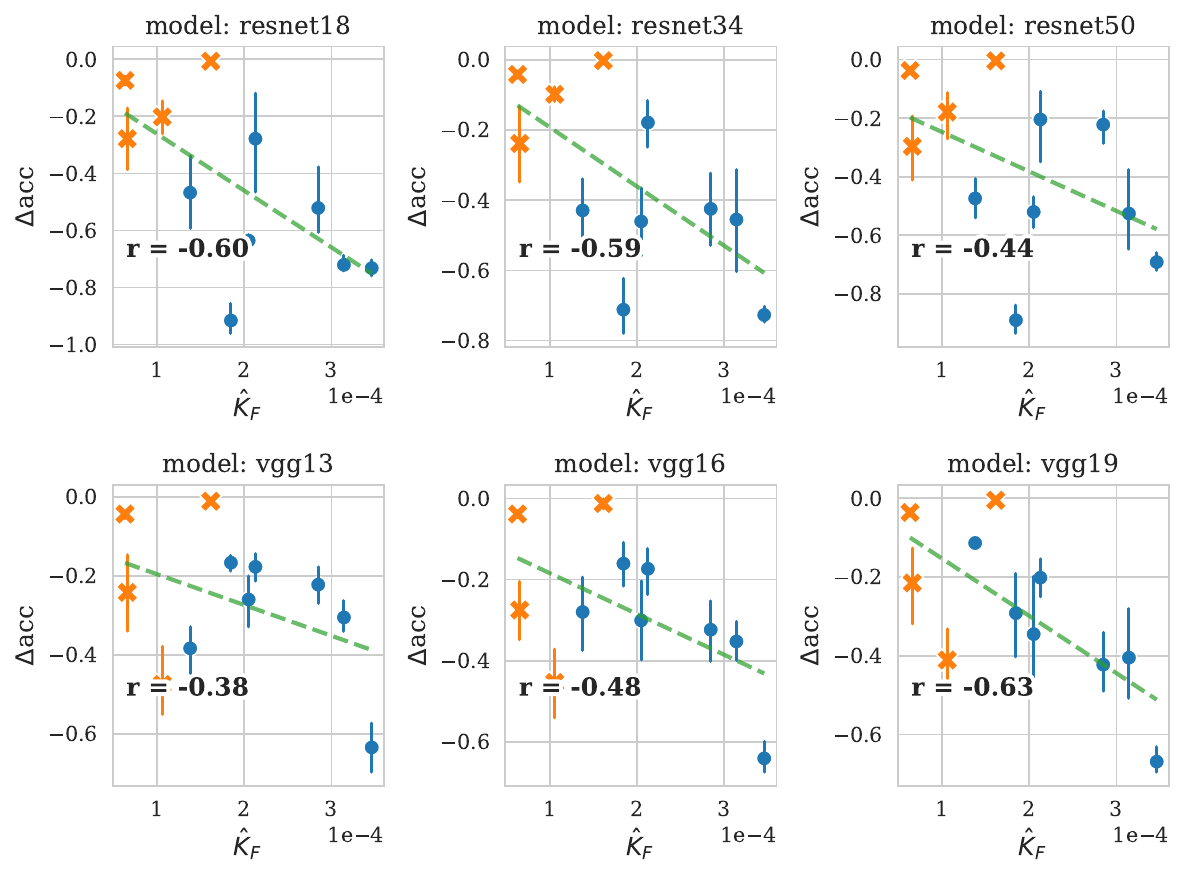}
    \caption{Scaling of test set accuracy penalty due to $\epsilon=1/255$ FGSM adversarial attack with dataset label sharpness $\KF$ for natural (orange) and medical (blue) datasets.}
    \label{fig:atk_vs_kf_acc_eps1}
\end{figure}

\begin{figure}[!htbp]
\centering
\includegraphics[width=0.7\textwidth]{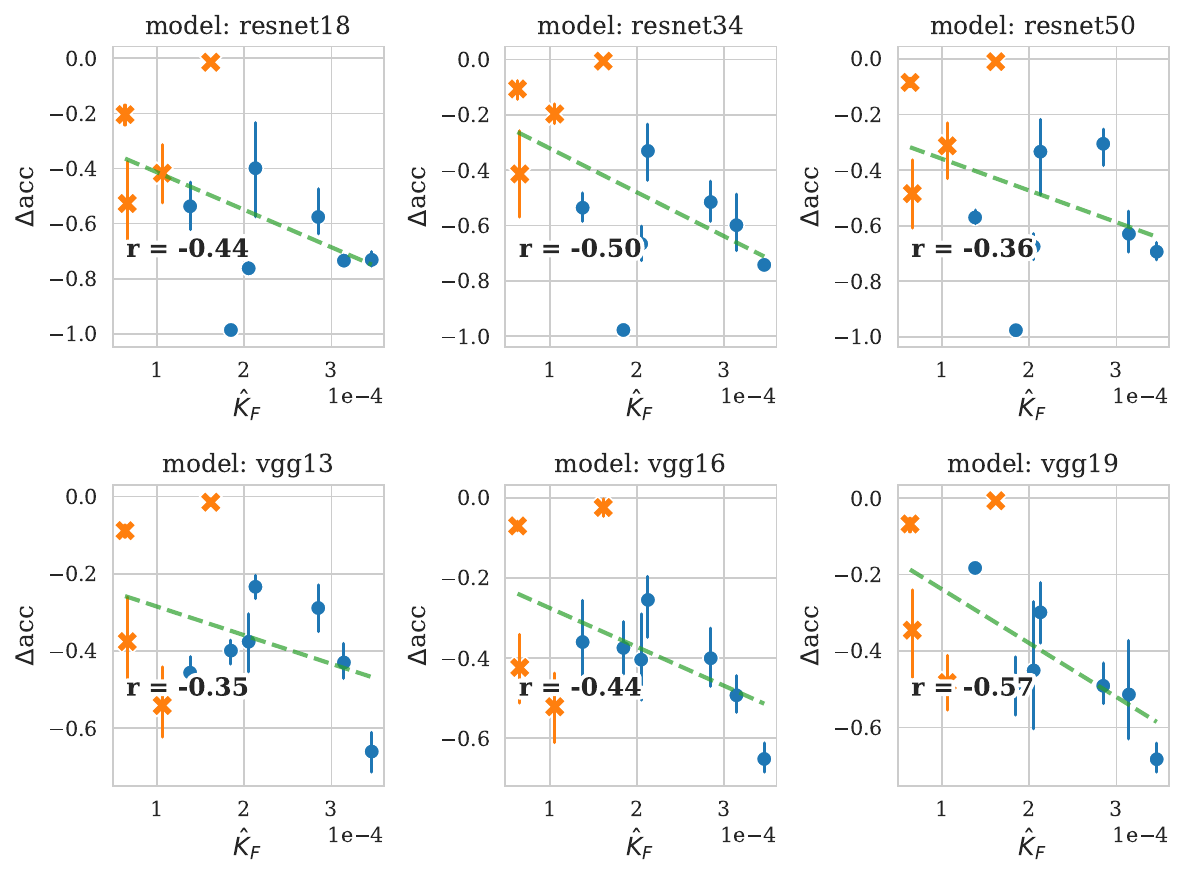}
    \caption{Scaling of test set accuracy penalty due to $\epsilon=2/255$ FGSM adversarial attack with dataset label sharpness $\KF$ for natural (orange) and medical (blue) datasets.}
    \label{fig:atk_vs_kf_acc_eps2}
\end{figure}

\begin{figure}[!htbp]
\centering
\includegraphics[width=0.7\textwidth]{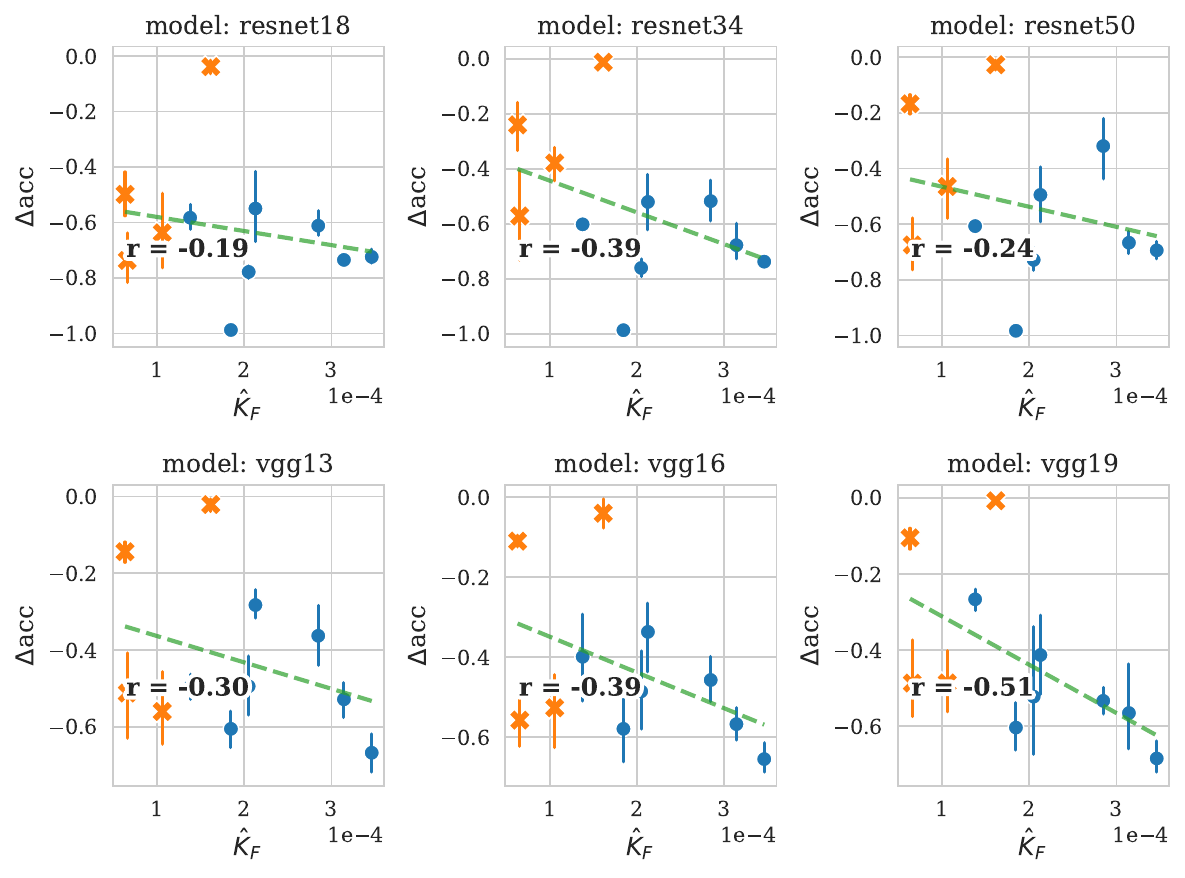}
    \caption{Scaling of test set accuracy penalty due to $\epsilon=4/255$ FGSM adversarial attack with dataset label sharpness $\KF$ for natural (orange) and medical (blue) datasets.}
    \label{fig:atk_vs_kf_acc_eps4}
\end{figure}

\begin{figure}[!htbp]
\centering
\includegraphics[width=0.7\textwidth]{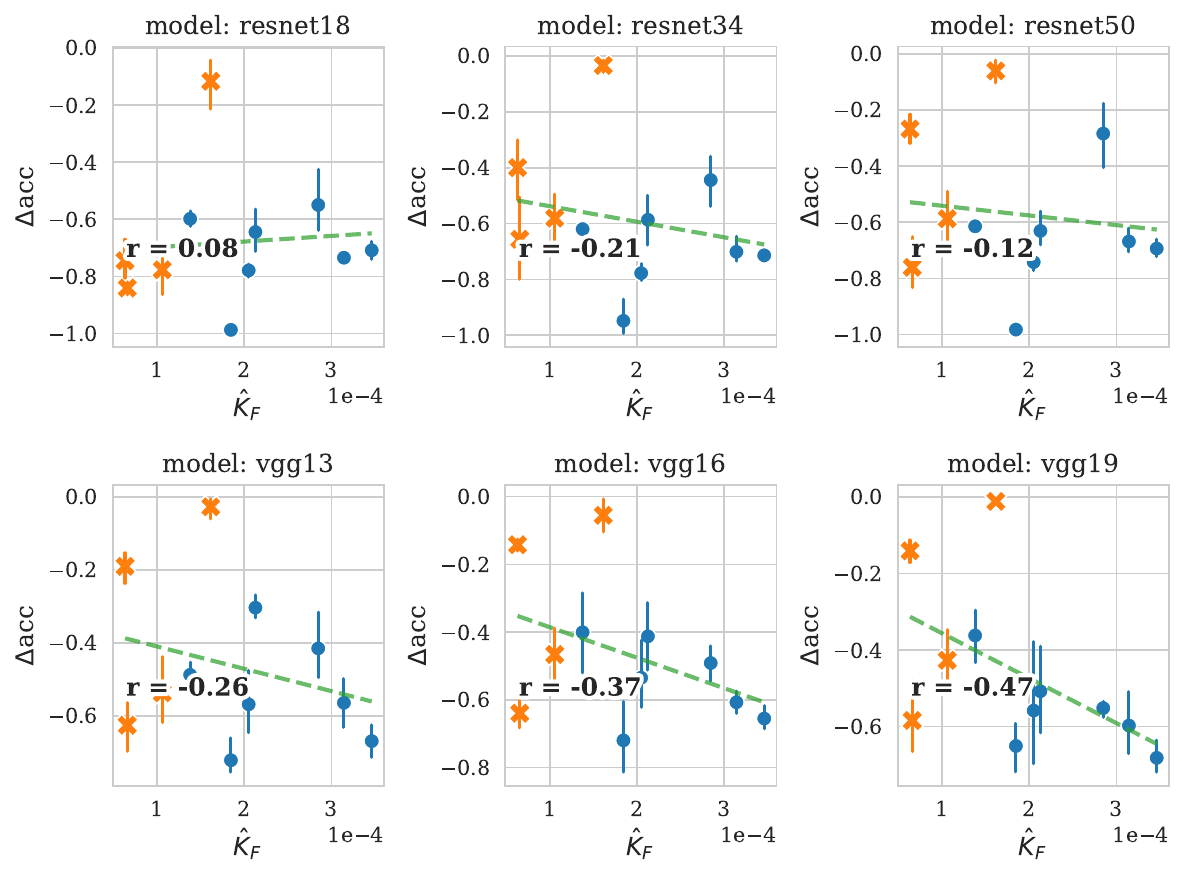}
    \caption{Scaling of test set accuracy penalty due to $\epsilon=8/255$ FGSM adversarial attack with dataset label sharpness $\KF$ for natural (orange) and medical (blue) datasets.}
    \label{fig:atk_vs_kf_acc_eps8}
\end{figure}

\begin{table}
\fontsize{9pt}{9pt}\selectfont
\centering
\begin{tabular}{l||ccc|ccc}
\hline
    Atk. $\epsilon$ & RN-18 & RN-34 & RN-50 & V-13 & V-16 & V-19 \\
    \hline
    $1/255$ & $0.67$ & $0.26$ & $0.43$ & $0.55$ & $0.69$ & $0.6$ \\
    $2/255$ & $0.53$ & $0.01$ & $0.28$ & $0.57$ & $0.71$ & $0.57$ \\
    $4/255$ & $0.41$ & $-0.16$ & $0.14$ & $0.56$ & $0.7$ & $0.53$ \\
    $8/255$ & $0.31$ & $-0.23$ & $0.04$ & $0.56$ & $0.66$ & $0.49$ \\
    \hline
    \end{tabular}
\caption{Pearson correlation $r$ between test loss penalty due to FGSM attack and dataset label sharpness $\KFh$, over all {\bf medical image} datasets and all training sizes. ``RN'' = ResNet, ``V'' = VGG.}
\label{tab:adv_medonly}
\end{table}

\begin{table}
\fontsize{9pt}{9pt}\selectfont
\centering
\begin{tabular}{l||ccc|ccc}
\hline
    Atk. $\epsilon$ & RN-18 & RN-34 & RN-50 & V-13 & V-16 & V-19 \\
    \hline
    $1/255$ & $-0.39$ & $-0.37$ & $-0.39$ & $-0.36$ & $-0.38$ & $-0.24$ \\
    $2/255$ & $-0.42$ & $-0.37$ & $-0.41$ & $-0.42$ & $-0.41$ & $-0.36$ \\
    $4/255$ & $-0.49$ & $-0.41$ & $-0.44$ & $-0.47$ & $-0.43$ & $-0.53$ \\
    $8/255$ & $-0.58$ & $-0.47$ & $-0.48$ & $-0.5$ & $-0.43$ & $-0.66$ \\
    \hline
    \end{tabular}
\caption{Pearson correlation $r$ between test loss penalty due to FGSM attack and dataset label sharpness $\KFh$, over all {\bf natural image} datasets and all training sizes. ``RN'' = ResNet, ``V'' = VGG.}
\label{tab:adv_natonly}
\end{table}

\section{Training and Implementational Details}
\label{app:implementation}

This section provides training and implementation details beyond that of Sec. \ref{sec:data}. We train all models with a binary cross-entropy loss function, optimize by Adam \citep{adam} with a weight decay strength of $10^{-4}$ for $100$ epochs. We use learning rates of $10^{-3}$ for ResNet models on all datasets, and $10^{-4}$ for VGG models on all datasets except SVHN, which required $10^{-6}$ to avoid loss divergence. ResNet-18, -34 and -50 models were trained with batch sizes of $200$, $128$, and $64$, respectively, and $32$ for all VGG models. We do not use any training image augmentations beyond resizing to $224\times224$ and linear normalization to $[0,1]$. We perform all experiments on a 48 GB NVIDIA A6000. %four 8 GB NVIDIA RTX 3070s.

\section{Medical Image Dataset Details}
\label{app:data}

This section goes into full detail into the binary classification task definitions for each medical image dataset, beyond what is mentioned in Section \ref{sec:data}. We follow the same task definitions for the medical image datasets as in \citet{konz2022intrinsic}. Specifically:
\begin{itemize}
    \item For OAI \citep{Tiulpin2018}, we use the screening packages 0.C.2 and 0.E.1, and define a negative class of X-ray images with Kellgren-Lawrence scores of 0 or 1, and a positive class of images with scores of 2+.
    \item For DBC \citep{saha2018machinedukedbc}, we use fat-saturated breast MRI slices. Slice images with a tumor bounding box label are positive, and any slice at least 5 slices away from a positive slice is negative.
    \item We use the same slice-labeling procedure as DBC for BraTS \citep{menze2014multimodal}, for glioma labels in T2 FLAIR brain MRI slices.
    \item For Prostate MRI \citep{sonn2013prostate}, we use slices from the middle 50\% of each MRI volume. Slices are labeled as negative if the volume's cancer risk score label is 0 or 1, and positive for 2+.
    \item For brain CT hemorrhage detection in RSNA-IH-CT \citep{flanders2020rsnaihct}, we detect for \textit{any} type of hemorrhage.
\end{itemize}

\end{document}